\documentclass[sigconf, screen]{acmart}

\copyrightyear{2025}
\acmYear{2025}
\setcopyright{rightsretained}
\acmConference[MM '25] {Proceedings of the 33rd ACM International Conference on Multimedia}{October 27--31, 2025}{Dublin, Ireland.}
\acmBooktitle{Proceedings of the 33rd ACM International Conference on Multimedia (MM '25), October 27--31, 2025, Dublin, Ireland}
\acmISBN{979-8-4007-2035-2/2025/10}
\acmDOI{10.1145/3746027.3755226}

\acmSubmissionID{2691}
\makeatother

\usepackage{algorithm}
\usepackage{algorithmic}
\usepackage{colortbl}
\usepackage{newfloat}
\usepackage{listings}

\usepackage[utf8]{inputenc} 
\usepackage[T1]{fontenc}    

\usepackage{url}            
\usepackage{booktabs}       
\usepackage{amsfonts}       
\usepackage{nicefrac}       
\usepackage{microtype}      
\usepackage{xcolor}         

 \usepackage{balance}

\usepackage{mathtools}
\usepackage{amsthm}
\usepackage{comment}

\usepackage{microtype}
\usepackage{graphicx}
\usepackage{caption,subcaption}
\usepackage{booktabs} 
\usepackage{booktabs}
\RequirePackage{algorithm}
\RequirePackage{algorithmic}
\usepackage{threeparttable}
\usepackage[capitalize,noabbrev]{cleveref}
\newtheorem{theorem}{Theorem}
\newtheorem{lemma}{Lemma}
\theoremstyle{remark}

\newtheorem{Assumption}{Assumption}

\theoremstyle{remark}

\begin{document}

\title{ Consistency of Local and Global Flatness for Federated Learning}

\author{Junkang Liu}
\affiliation{%
	\institution{College of Intelligence and Computing, Tianjin University}
	\city{Tianjin}
	\country{China}
}
\email{junkangliukk@gmail.com}

\author{Fanhua Shang}

\authornote{Corresponding authors}
\affiliation{%
	\institution{College of Intelligence and Computing, Tianjin University}
	\city{Tianjin}
	\country{China}
}
\email{fhshang@tju.edu.cn}

\author{Yuxuan Tian}
\affiliation{%
	\institution{College of Management and Economics, Tianjin University}
	\city{Tianjin}
	\country{China}
}
\email{guilang@tju.edu.cn}

\author{Hongying Liu}
\authornotemark[1]
\affiliation{%
	\institution{Medical College, Tianjin University\\Peng Cheng Laboratory}
	\city{Tianjin}
	\country{China}
}
\email{hyliu2009@tju.edu.cn}

\author{Yuanyuan Liu}
\authornotemark[1]
\affiliation{%
	\institution{School of Artificial Intelligence, Xidian University}
	\city{Xian}
	\country{China}  
}
\email{yyliu@xidian.edu.cn}



\renewcommand{\shortauthors}{Junkang Liu, Fanhua Shang, Yuxuan Tian,  Hongying Liu, \& Yuanyuan Liu}

\begin{abstract}
In federated learning (FL), multi-step local updates and data heterogeneity usually lead to sharper global minima, which degrades the performance of the global model. Popular FL algorithms integrate sharpness-aware minimization (SAM) into local training to address this issue. However, in the high data heterogeneity setting, the flatness in local training does not imply the flatness of the global model. Therefore, minimizing the sharpness of the local loss surfaces on the client data does not enable the effectiveness of SAM in FL to improve the generalization ability of the global model. We define the \textbf{flatness distance} to explain this phenomenon. By rethinking the SAM in FL and theoretically analyzing the \textbf{flatness distance}, we propose a novel \textbf{FedNSAM} algorithm that accelerates the SAM algorithm by introducing global Nesterov momentum into the local update to harmonize the consistency of global and local flatness. \textbf{FedNSAM} uses the global Nesterov momentum as the direction of local estimation of client global perturbations and extrapolation. Theoretically, we prove a tighter convergence bound than FedSAM by Nesterov extrapolation. Empirically, we conduct comprehensive experiments on CNN and Transformer models to verify the superior performance and efficiency of \textbf{FedNSAM}. The code is available at \url{https://github.com/junkangLiu0/FedNSAM}.
\end{abstract}

\ccsdesc[500]{Theory of computation~Massively parallel algorithms}

\keywords{Federated Learning, Sharpness-Aware Minimization, Generalization Ability}



\maketitle

\section{Introduction}
\begin{figure}[tb]
	\begin{minipage}[b]{0.155\textwidth}
		\centering
		\subcaptionbox{ FedSAM (0.6)}{\includegraphics[width=\textwidth]{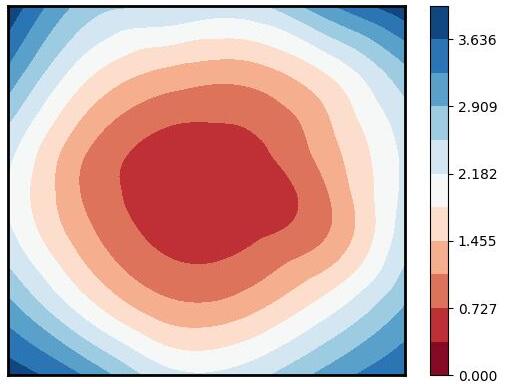}}
	\end{minipage}
	\begin{minipage}[b]{0.155\textwidth}
		\centering
		\subcaptionbox{ FedSAM (0.1)}{\includegraphics[width=\textwidth]{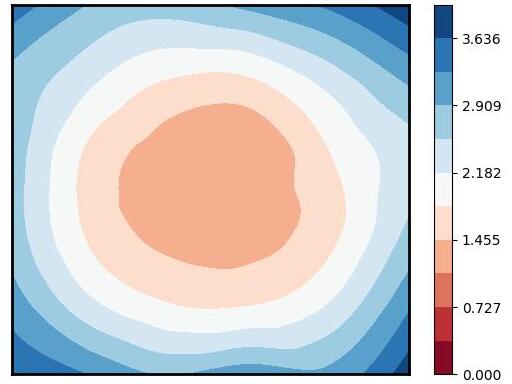}}
	\end{minipage}
	\begin{minipage}[b]{0.155\textwidth}
		\subcaptionbox{ \textbf{FedNSAM} (0.1)}{\includegraphics[width=\textwidth]{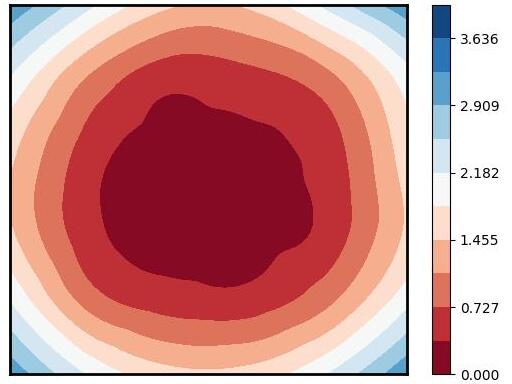}}
	\end{minipage}
	\begin{minipage}[b]{0.155\textwidth}
		\subcaptionbox{IID}{\includegraphics[width=\textwidth]{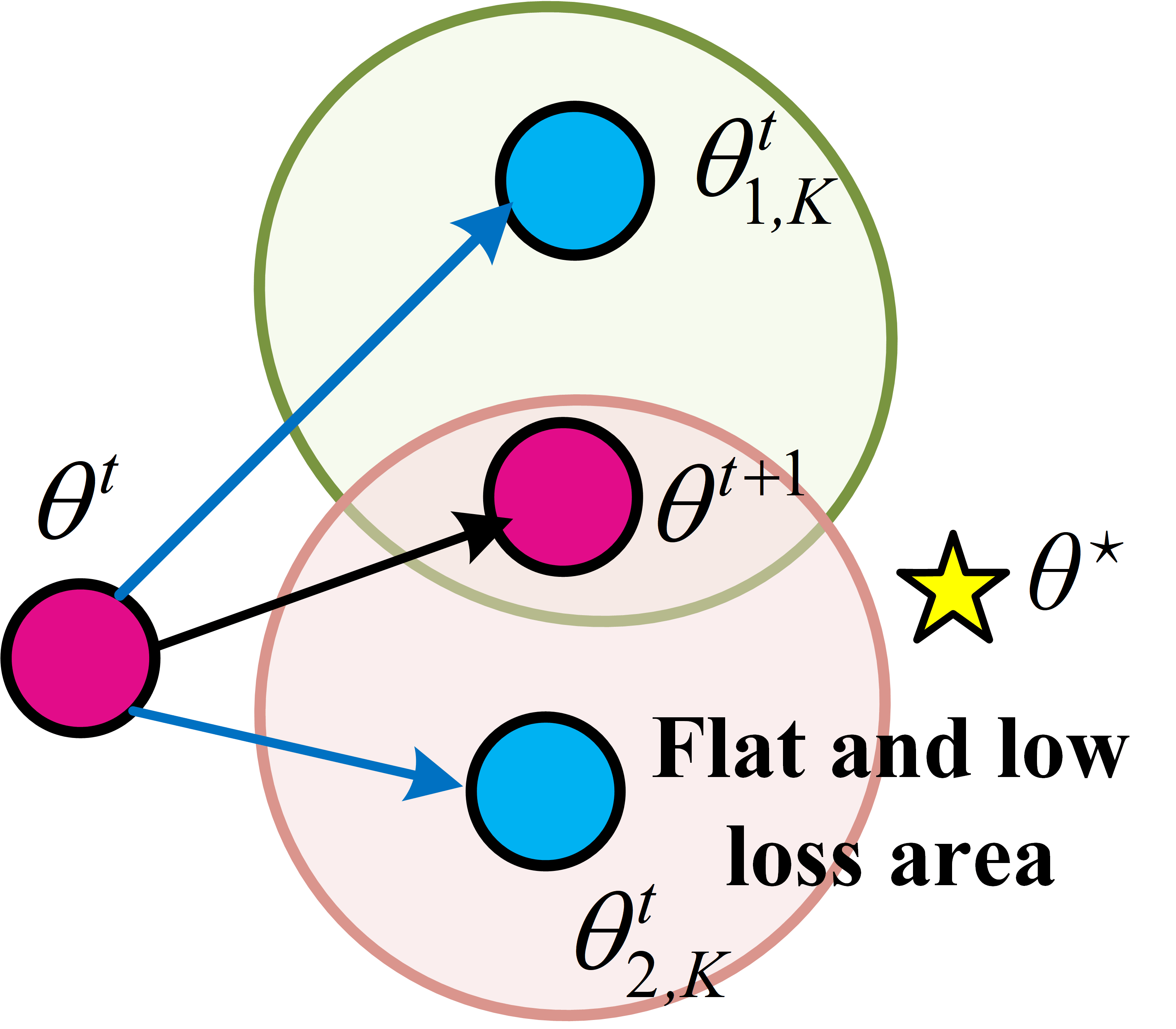}}
	\end{minipage}
	\begin{minipage}[b]{0.155\textwidth}
		\subcaptionbox{Non-IID}{\includegraphics[width=\textwidth]{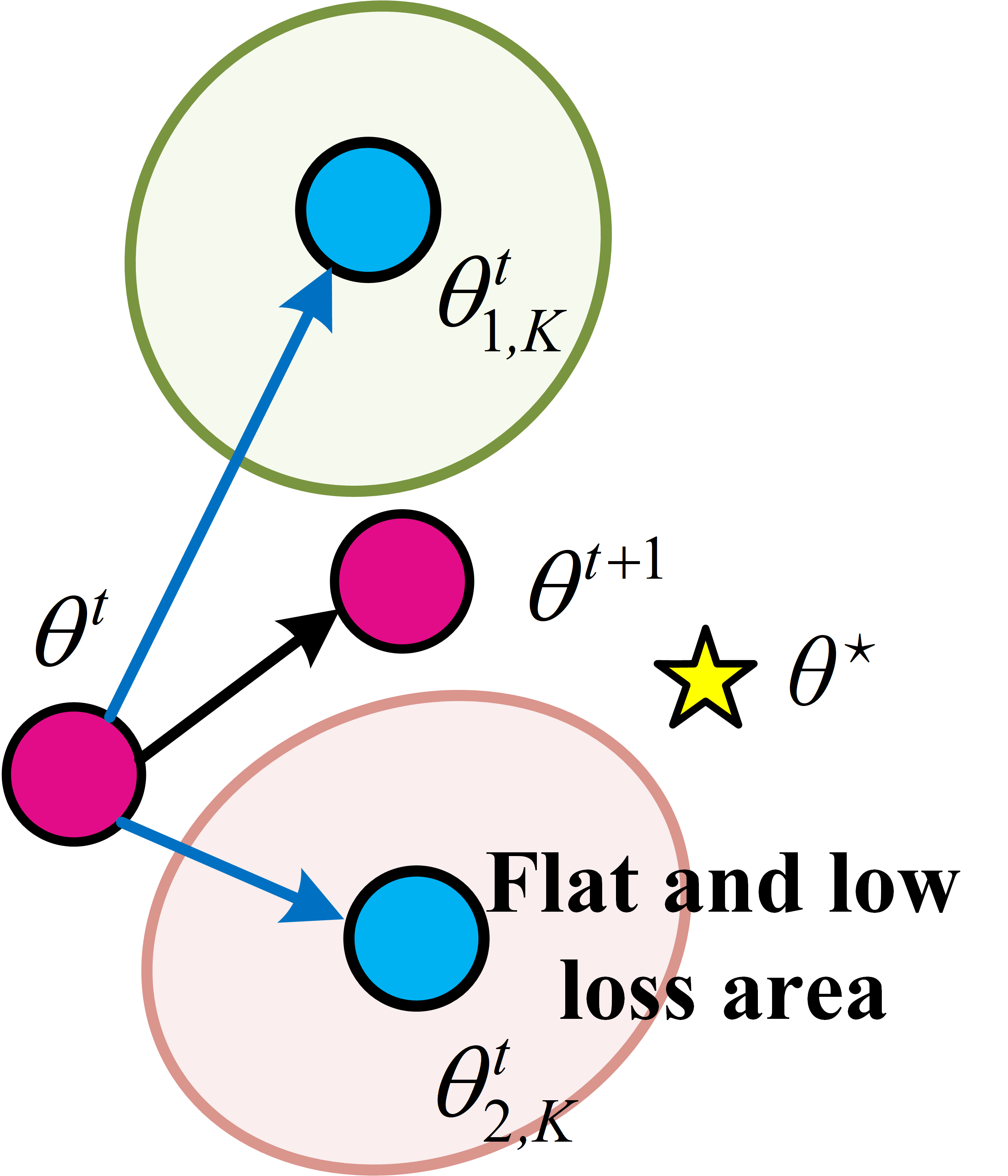}}
	\end{minipage}
	\begin{minipage}[b]{0.155\textwidth}
		\subcaptionbox{Correction}{\includegraphics[width=\textwidth]{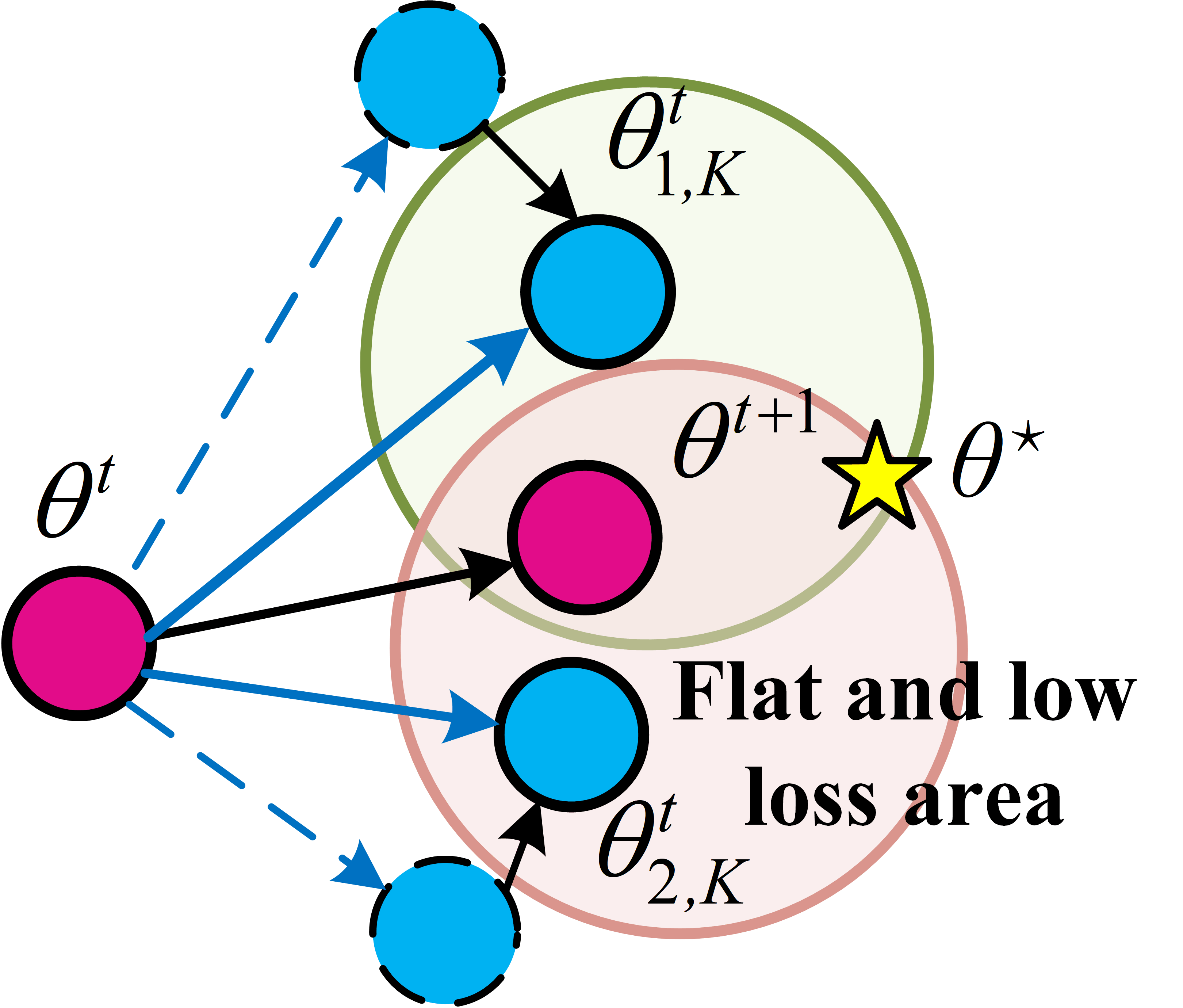}}
	\end{minipage}
	\caption{(a) and (b) show the global training loss surface of FedSAM \cite{qu2022generalized} under Dirichlet distributions with coefficients of 0.6 (low data heterogeneity) and 0.1 (high data heterogeneity) on CIFAR-100 with ResNet-18. (c) shows the training loss surface of our \textbf{FedNSAM} with Dirichlet-0.1. FedSAM can search global flat minima in low data heterogeneity but fails in high data heterogeneity. (d) suggests that the flat region of the client is closer in low data heterogeneity setting. The global model falls within the flat region. (e) suggests that the flat region of the client is far in high data heterogeneity setting. The global model cannot fall within the flat region. (f) suggests that our \textbf{FedNSAM} draws the flat region between each client closer by alignment correction so that the global model falls within the flat region of each client. \textbf{FedNSAM}'s global model finds flat minima in (c).	}
	\label{figure 1}
\end{figure}

Federated Learning (FL) has garnered significant interest as a crucial framework for decentralized training among numerous clients while ensuring data privacy. The core principle of FL involves maintaining local data on the client, allowing only the communication of gradients or model parameters between clients and a central server, and the server does not have direct access to local data \cite{mcmahan2017communication}. FL has considerable promise in various regions, including healthcare, finance, and personalized mobile services \cite{rieke2020future,antunes2022federated,byrd2020differentially,liu2024fedbcgd,zeng2025FSDrive,dai2025securetugofwarsectowiterative,dai2025captionsrewardscarevlleveraging,liuimproving,yang2025distillation}. 

Recently, there has been widespread interest in how to improve generalization for FL \cite{caldarola2022improving,sun2023dynamic,qu2022generalized}. However, \cite{fan2024locally} shows that data heterogeneity and multi-step local updates can cause the global model to converge to sharp local minima, which tend to exhibit poor generalization. A popular approach to this problem is to use sharpness-aware minimization (SAM) \cite{foret2020sharpness} as a local optimizer to find flat local minima, such as FedSAM \cite{qu2022generalized}. However, in the setting of highly heterogeneous data, we find that FedSAM is effective in finding flat minima in local training, but the aggregated global model is not a flat minima (see Figure \ref{figure 1} (a,b)). The literature \cite{fan2024locally} also found the same problem as us, but it did not provide a good explanation for the reason behind this phenomenon.

As illustrated in Figure~\ref{figure 1}, in the low heterogeneity case (d), clients share similar update directions, resulting in overlapping flat regions. This increases the likelihood that the global model converges to a shared flat minimum, as shown in (a). In contrast, in the high heterogeneity case (e), client updates diverge significantly, making their flat regions disjoint and preventing the global model from residing in any of them—thus leading to sharp global minima.

To explain this phenomenon, we introduce the concept of \textbf{flatness distance}, which quantifies the discrepancy between local and global flatness. As shown in Figure~\ref{figure 21}, higher heterogeneity leads to greater \textbf{flatness distance} and degraded global performance. To mitigate this, we propose \textbf{FedNSAM}, which aligns client flat regions through Nesterov momentum correction, enabling the global model to fall within them, as visualized in Figure~\ref{figure 1}(f).



\textbf{Contributions.} 
We identify key limitations of the SAM algorithm in federated learning (FL) and propose a new perspective—\textbf{flatness distance}—to quantify the divergence of flat regions across clients under data heterogeneity. To address these challenges, we design a new FL algorithm: \textbf{FedNSAM}, which leverages Nesterov momentum to minimize global sharpness by aligning local flat regions. Our main contributions are summarized as follows:

\textbf{$\bullet$} We introduce the concept of \textbf{flatness distance} to characterize the inconsistency of local minima across clients. Through both theoretical and empirical analysis, we demonstrate that higher data heterogeneity leads to increased \textbf{flatness distance} and degraded global flatness.

\textbf{$\bullet$} We propose a novel sharpness-aware federated learning algorithm (\textbf{FedNSAM}) with Nesterov extrapolation at the client level. We theoretically prove its convergence rate of $O\big(\sqrt{L F}/\sqrt{T K S(1-\lambda)}\big)$ and show that it achieves lower flatness distance than FedSAM.

\textbf{$\bullet$} Extensive experimental results on three benchmark datasets validate that \textbf{FedNSAM} outperforms existing methods under various data heterogeneity levels, participation rates, and model architectures, demonstrating superior generalization and optimization performance.

\section{Related work}
\textbf{$\bullet$ Heterogeneity Issues in FL:} In past years, various strategies have been proposed to solve the heterogeneity issues in FL. A main branch of prior approaches focuses on regularizing local training to alleviate the divergence of the local models. As early works, FedAvgM \cite{hsu2019measuring} is a proposed method that introduces momentum terms during global model updating. FedACG \cite{kim2024communication} improves the inter-client interoperability by the server broadcasting a global model with a prospective gradient consistency. SCAFFOLD \cite{karimireddy2020scaffold} uses SAGA-like control variables to mitigate client-side drift, which can be regarded as adopting the idea of variance reduction at the client side. FedProx \cite{li2018federated} reduces the variance of updates from different clients by adding a proximal term to the objective function to limit the magnitude of updates from each client.
FedBCGD \cite{liu2024fedbcgd} proposes an accelerated block coordinate gradient descent framework for FL.
FedSWA \cite{liuimproving} improves generalization under highly heterogeneous data by stochastic weight averaging.
FedAdamW \cite{liu2025fedadamw} introduces a communication-efficient AdamW-style optimizer tailored for federated large models.
FedNSAM \cite{liu2025consistency} studies the consistency relationship between local and global flatness in FL.
FedMuon \cite{liu2025fedmuon} accelerates federated optimization via matrix orthogonalization.
DP-FedPGN \cite{liu2025dp} develops a penalizes gradient norms to encourage globally flatter minima in DP-FL.
FedPAC \cite{liu2026tamingpreconditionerdriftunlocking} mitigates preconditioner drift to unlock the potential of second-order optimizers.
Our
\textbf{FedNSAM} method is based on flat minima searching, which differs from the above methods.\\
\textbf{$\bullet$ Improve Generalization in FL:}
The existing FL methods that search flatter minima for better generalization utilize Sharpness Aware Minimization (SAM) \cite{foret2020sharpness}  or its variants to local training. \cite{qu2022generalized,caldarola2022improving} apply SAM optimizer in FL local training and proposed FedSAM. \cite{qu2022generalized} proposed a
variant of FedSAM called MoFedSAM with global momentum in local training. \cite{dai2023fedgamma} proposed FedGAMMA, which
reduces the effect of data heterogeneity by using SCAFFOLD \cite{karimireddy2020scaffold} control variables in FedSAM. 
FedSMOO \cite{sun2023dynamic} attempted to address this by correcting local perturbations, but it introduces many computational overheads as other SAM-based algorithms. \cite{fan2024locally} propose FedLESAM locally estimates the direction of global perturbation on the client side as the difference between global models received in the previous active and current rounds. 


\section{Rethinking SAM in FL}
\subsection{FL Problem Setup}
FL aims to optimize model parameters with  local clients, i.e., minimizing the following population risk:
\vspace{-2mm}
\begin{align}
	F(\boldsymbol{\boldsymbol{\theta}})=\frac{1}{N} \sum_{i=1}^N\left(F_i(\boldsymbol{\boldsymbol{\theta}}):=\mathbb{E}_{\zeta_i\sim \mathcal{D}_i}\left[F_i\left(\boldsymbol{\boldsymbol{\theta}} ; \zeta_i\right)\right]\right).
	\label{eq 1}
\end{align}
The function $F_i$ denotes the loss on client $i$, where $\zeta_i \sim \mathcal{D}_i$ is a sample from its local data distribution. The expectation $\mathbb{E}_{\zeta_i}[\cdot]$ is taken over $\zeta_i$. The total number of clients is $N$.

\subsection{Flatness Searching in FL: FedSAM}
The optimization of FedSAM \cite{qu2022generalized} is formulated as follows:
\vspace{-2mm}
\begin{align}
	\min _\theta \max _{\left\|\boldsymbol{\delta}_i\right\|_2^2 \leq \rho}\left\{F(\tilde{\boldsymbol{\boldsymbol{\theta}}}):=\frac{1}{N} \sum_{i \in[N]} F_i(\tilde{\boldsymbol{\boldsymbol{\theta}}})\right\},
\end{align}
where $F(\tilde{\boldsymbol{\boldsymbol{\theta}}}) \triangleq \max _{\|\boldsymbol{\delta}\| \leq \rho} F(\boldsymbol{\boldsymbol{\theta}}+\boldsymbol{\delta}), \quad F_i(\tilde{\boldsymbol{\boldsymbol{\theta}}}) \triangleq$ $\max _{\left\|\boldsymbol{\delta}_i\right\| \leq \rho} F_i\left(\boldsymbol{\boldsymbol{\theta}}+\boldsymbol{\delta}_i\right)$.
The local updates of FedSAM include
\vspace{-2mm}
\begin{equation}
\begin{aligned}
	& \boldsymbol{\boldsymbol{\theta}}_{i, k+1/2}^t=\boldsymbol{\boldsymbol{\theta}}_{i, k}^t+\boldsymbol{\delta}_{i, k}^t,\quad
	\boldsymbol{\delta}_{i, k}^t=\rho g_{i, k}^t /\left\|g_{i, k}^t\right\|,\\	&\boldsymbol{\boldsymbol{\theta}}_{i, k+1}^t=\boldsymbol{\boldsymbol{\theta}}_{i, k}^t-\eta g_{i, k+1/2}^t,
\end{aligned}
\end{equation}
where $g_{i, k}^t=\nabla F_i(\boldsymbol{\boldsymbol{\theta}}_{i, k}^t; \zeta_i)$ is the stochastic gradient computed at $\boldsymbol{\boldsymbol{\theta}}_{i, k}^t$, and $ g_{i, k+1/2}^t=\nabla F_i(\boldsymbol{\boldsymbol{\theta}}_{i, k+1/2}^t; \zeta_i)$, which can be seen in Figure \ref{figure 2}  (a). FedSAM searches for the  flat local minima with local loss sharpness-aware minimization instead of the  flat global minima.

\begin{figure}[tb]
	\begin{minipage}[b]{0.235\textwidth}
		\centering
		\subcaptionbox{\textbf{flatness distance}}{\includegraphics[width=\textwidth]{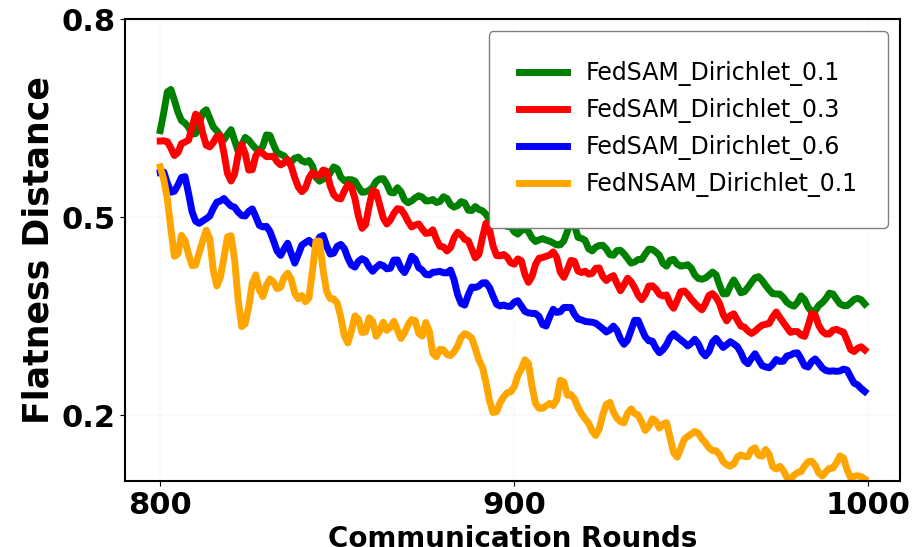}}
	\end{minipage}
	\begin{minipage}[b]{0.23\textwidth}
		\centering
		\subcaptionbox{Global Sharpness}{\includegraphics[width=\textwidth]{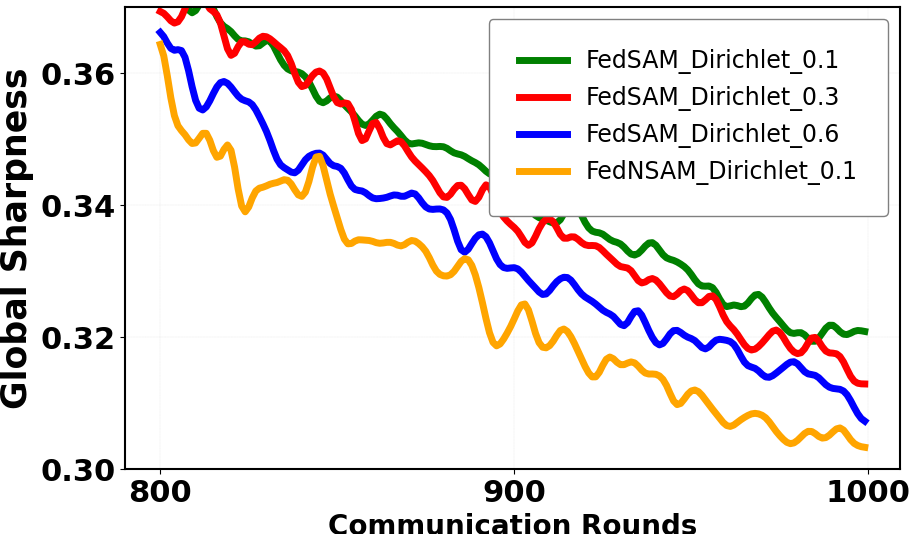}}
	\end{minipage}
	\caption{Illustration of \textbf{flatness distance} (left) and global sharpness (right) during federated training on CIFAR100, using Dirichlet distributions with coefficients 0.1, 0.3, and 0.6 across 100 clients and a 10\% participation rate. Test accuracies are 40.18\%, 46.02\%, and 47.83\% for FedSAM with Dirichlet-0.1, 0.3, and 0.6 respectively, and 58.53\% for \textbf{FedNSAM} with Dirichlet-0.1.
}
	\label{figure 21}
\end{figure}
\subsection{ \textbf{Flatness Distance}}

We formally define the \textbf{flatness distance}  $\boldsymbol{\Delta}_{\mathcal{D}}$, as the difference gap of the flatness between the global and local models.\\
\textbf{Definition 1 \textit{\textbf{(Flatness distance)}}.} \textit{	 $\boldsymbol{\Delta}_{\mathcal{D}}$ of the global model
	$\boldsymbol{\theta}^{t+1}=\frac{1}{N} \sum_{i=1}^N \boldsymbol{\theta}_{i, K}^t$ and the local models $\left\{\boldsymbol{\theta}_{i, K}^t\right\}_{i=1}^N$ is defined as:}
\vspace{-2mm}
\begin{align}
\boldsymbol{\Delta}_{\mathcal{D}}=\frac{1}{N} \sum_{i=1}^N \mathbb{E}\left\|\boldsymbol{\theta}_{i, K}^{t}-\boldsymbol{\theta}^{t+1}\right\|^2.
\end{align}
When $\boldsymbol{\Delta}_{\mathcal{D}}$ is small, the flat regions across clients are closer, leading to higher global flatness. Conversely, a larger $\boldsymbol{\Delta}_{\mathcal{D}}$ indicates greater disparity between client flatness regions, resulting in degraded global flatness.

As shown in Figure~\ref{figure 21}, FedSAM suffers significant performance degradation as data heterogeneity increases, which corresponds to larger \textbf{flatness distance} and sharper global minima. In contrast, the proposed \textbf{FedNSAM} algorithm consistently achieves lower \textbf{flatness distance} and global sharpness under high heterogeneity, thanks to Nesterov momentum correction. Theoretical justifications are provided in Section~5.2.

\subsection{ Global Sharpness-Aware Minimization in FL}

To get global flat minima, we first recall the definition of global sharpness-aware minimization \cite{fan2024locally} in FL:
\vspace{-2mm}
\begin{align}
	\min _\theta \max _{\|\boldsymbol{\delta}\|_2 \leq \rho}\left\{F(\boldsymbol{\theta}+\boldsymbol{\delta})=\frac{1}{N} \sum_{i=1}^N F_i(\boldsymbol{\theta}+\boldsymbol{\delta})\right\},
\end{align}
the  global perturbation $\boldsymbol{\delta}_{k}^t$ at the $k$-th iteration in $t$-th round is formulated as follows:
\vspace{-2mm}
\begin{align}
	\boldsymbol{\delta}_{k}^t=\rho \frac{\nabla F\left(\boldsymbol{\boldsymbol{\theta}}_{k}^t\right)}{\left\|\nabla F\left(\boldsymbol{\boldsymbol{\theta}}_{k}^t\right)\right\|}=\rho \frac{\frac{1}{N}\sum_{i=1}^N \nabla F_i\left(\boldsymbol{\boldsymbol{\theta}}_{ k}^t\right)}{\left\|\frac{1}{N}\sum_{i=1}^N \nabla F_i\left(\boldsymbol{\boldsymbol{\theta}}_{k}^t\right)\right\|},
\end{align}
where $\boldsymbol{\boldsymbol{\theta}}_{k}^t=\boldsymbol{\boldsymbol{\theta}}^t- \frac{1}{N} \sum_{i=0}^N\left(\boldsymbol{\boldsymbol{\theta}}^t-\boldsymbol{\boldsymbol{\theta}}_{i, k}^t\right)$ is the virtual global model. However, we can neither share weights nor gradients of clients during the local training in FL. And we have no access to the global gradient $\nabla F\left(\boldsymbol{\boldsymbol{\theta}}_{k}^t\right)$, so we use the global momentum $\boldsymbol{m}^t$ to estimate the global gradient. Therefore, the global perturbation can be approximately calculated as follows:
\begin{align}
	\boldsymbol{\delta}_{ k}^t=\rho \frac{\nabla F\left(\boldsymbol{\theta}_{k}^t\right)}{\left\|\nabla F\left(\boldsymbol{\theta}_{k}^t\right)\right\|} \approx \rho \frac{\boldsymbol{m}^t}{\left\|\boldsymbol{m}^t\right\|},
\end{align}
where $\boldsymbol{m}^t=\lambda \boldsymbol{m}^{t-1}+\boldsymbol{\Delta}^t$, $\boldsymbol{\Delta}^t=\frac{1}{S}\sum_{i \in S_t} \boldsymbol{\Delta}_i^t $, $\boldsymbol{\Delta}_i^t = \boldsymbol{\theta}_{i, K}^t-\boldsymbol{\theta}_{i, 0}^t$. 
Next, we can get the perturbed local model,
\vspace{-2mm}
\begin{align}
	& \boldsymbol{\theta}_{i, k+1/2}^t=\boldsymbol{\theta}_{i,k}^t+\boldsymbol{\delta}_{ k}^t,
\end{align}
Finally, we compute the global SAM gradient  at the perturbed model $\boldsymbol{\theta}_{i, k+1/2}^t$ to update $\boldsymbol{\theta}_{i, k}^t$,
\vspace{-2mm}
\begin{align}
	& \boldsymbol{\theta}_{i, k+1}^t=\boldsymbol{\theta}_{i, k}^t-\eta \nabla F\left(\boldsymbol{\theta}_{k+1/2}^t\right),
\end{align}
where $\nabla F\big(\boldsymbol{\theta}_{k+1/2}^t\big)=\frac{1}{N}\sum_{i=1}^N \nabla F_i\big(\boldsymbol{\theta}_{k+1/2}^t\big)$.  However, the computation of the global SAM gradient is impractical in FL. Due to data heterogeneity and local multi-iteration, $\nabla F_i\big(\boldsymbol{\theta}_{k+1 / 2}^t\big) \neq \nabla F\big(\boldsymbol{\theta}_{k+1 / 2}^t\big)$. To reduce the inconsistency between the local models and the global model, we incorporate global momentum $\boldsymbol{m}_t$ into the local models to guide local
updates in our \textbf{FedNSAM}.

\definecolor{LightRed}{RGB}{255,182,193} 
\definecolor{LightBlue}{RGB}{173, 216, 230}
\begin{table*}[htbp]
	\centering
	\caption{Summary of federated SAM-based algorithms for addressing data heterogeneity, focusing on  perturbation  correction, updating  correction,  acceleration, and extra computation introduced by SAM. Here, perturbation correction is a correction for $\boldsymbol{\delta}_{i,k}^t$, updating  correction is a correction for $\nabla F_i(\boldsymbol{\theta}_{k+1/2}^t)$, and Com Cost denotes the communication cost.} 
	\vspace{-4mm}
	\label{table 1}
	\setlength{\tabcolsep}{6.5pt}	\begin{tabular}{llcccccc}
		\midrule[1.5pt]
		\centering
		& \textbf{Research work}   &  \textbf{Perturbation  Correction}  &  \textbf{Updating  correction} &\textbf{Acceleration}& \textbf{Less Computation} & \textbf{Com Cost}\\
		\midrule[0.5pt]		
		&FedSAM \cite{qu2022generalized}& $\times$ & $\times$ & $\times$ &$\times$&1$\times$\\		
		&MoFedSAM \cite{qu2022generalized} & $\times$ & \checkmark &  $\times$ &$\times$&1$\times$\\	
		&FedGAMMA \cite{dai2023fedgamma} & $\times$ & \checkmark &  $\times$&$\times$&2$\times$ \\		
		&FedSMOO  \cite{sun2023dynamic} & \checkmark & $\times$ & $\times$ & $\times$&2$\times$ \\
		&FedLESAM \cite{fan2024locally} & \checkmark & $\times$ &  $\times$& \checkmark&1$\times$\\
        \rowcolor{LightRed}
		&\textbf{\textbf{FedNSAM}} (ours) & \checkmark & \checkmark &  \checkmark& \checkmark&1$\times$\\
		\midrule[1.5pt]
	\end{tabular}
	\label{table 3}
\end{table*}

\section{Proposed Algorithm: \textbf{FedNSAM}}
The main idea of \textbf{FedNSAM} is to accelerate the local update using Nesterov momentum that utilizes the global momentum $\boldsymbol{m}_t$ to guide the local update, and to use the  $\boldsymbol{m}_t$ as the perturbation direction of the SAM algorithm, so that each client can look for a more consistent flat region.

\begin{figure}[tb]
	\begin{minipage}[b]{0.225\textwidth}
		\centering
		\subcaptionbox{FedSAM}{\includegraphics[width=\textwidth]{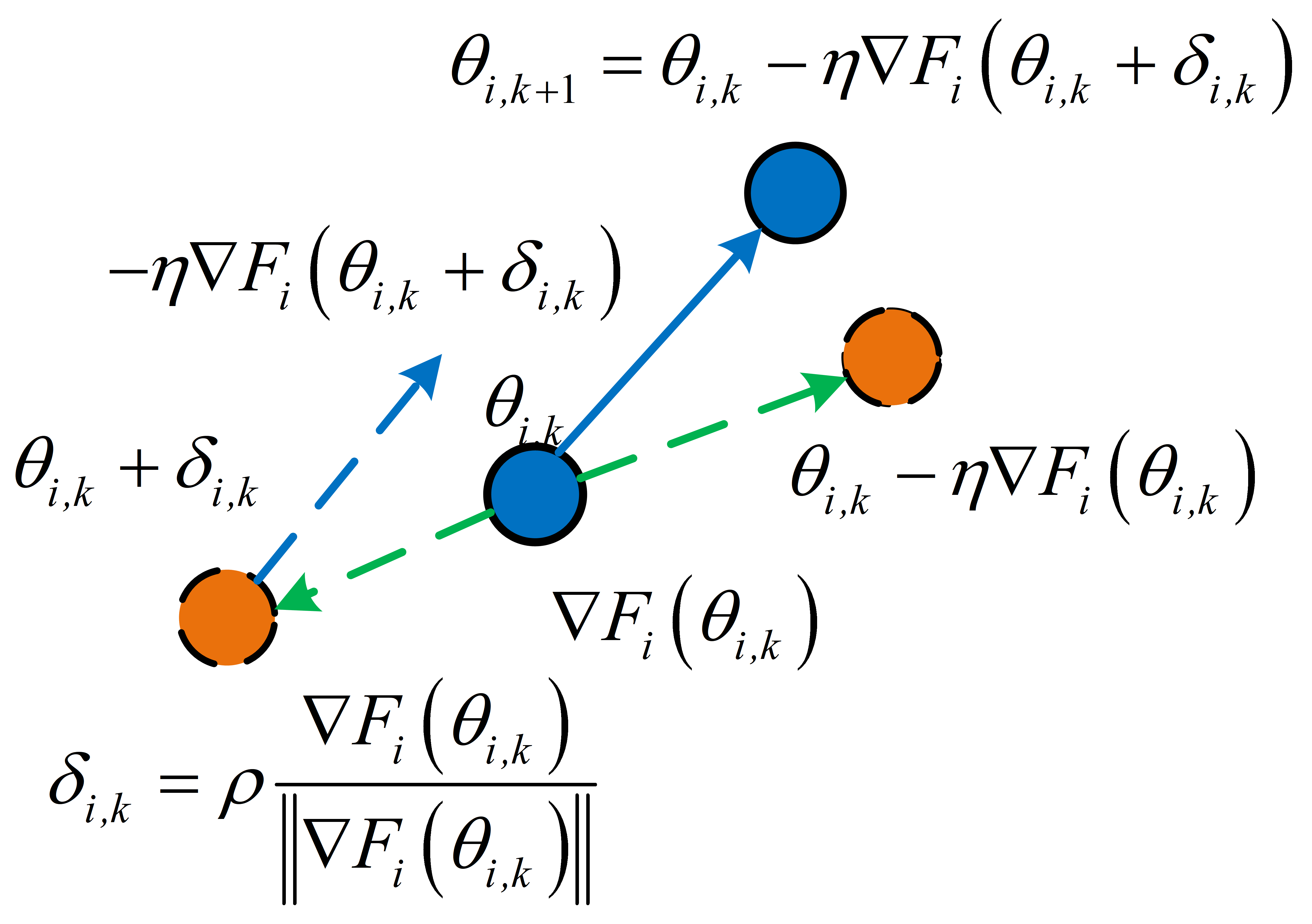}}
	\end{minipage}
	\begin{minipage}[b]{0.225\textwidth}
		\centering
		\subcaptionbox{\textbf{FedNSAM}}{\includegraphics[width=\textwidth]{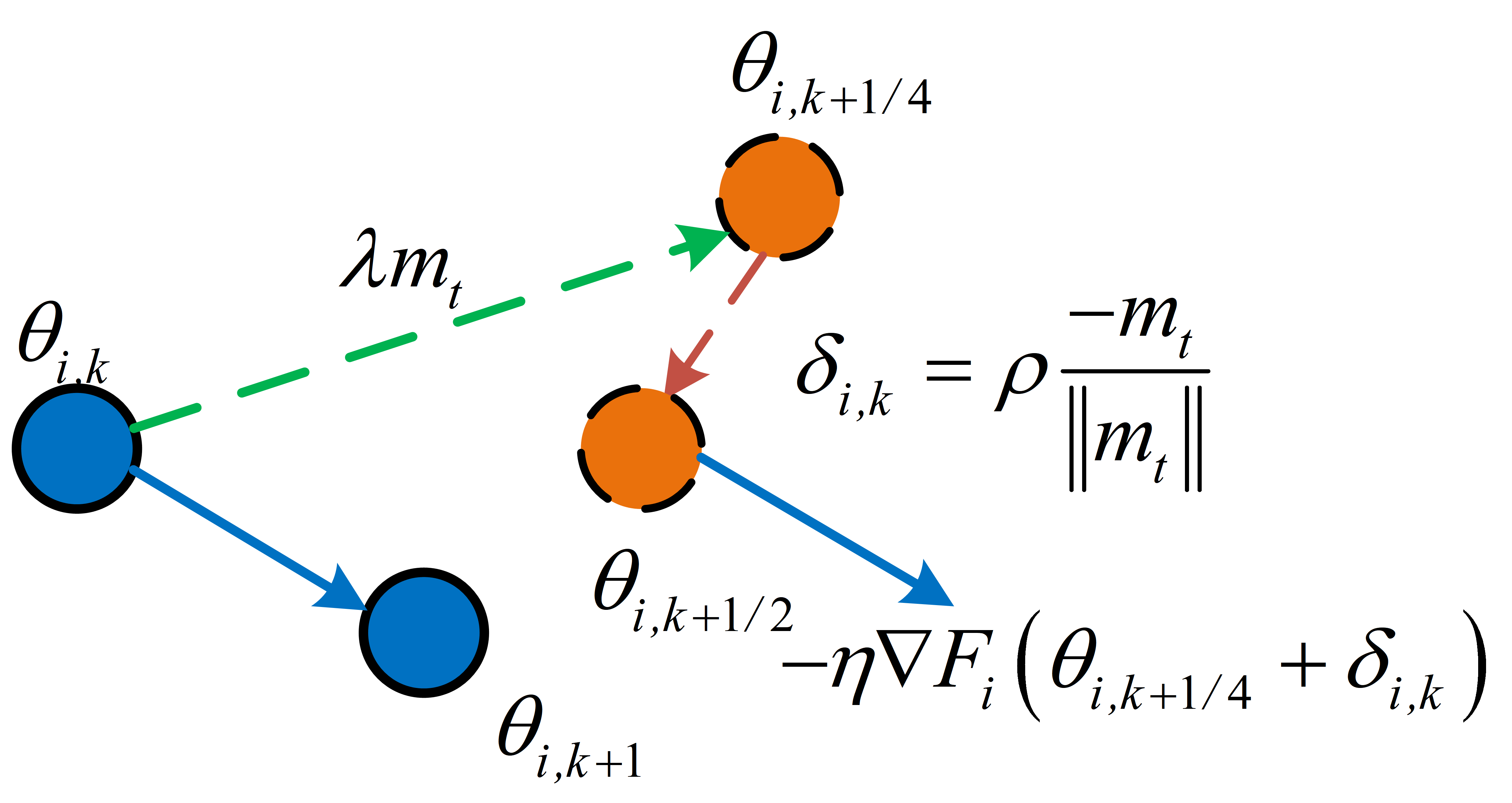}}
	\end{minipage}
	\caption{(a) and (b) depict the local update procedures of FedSAM and \textbf{FedNSAM}, respectively. FedNSAM incorporates global Nesterov momentum into local sharpness-aware updates to improve alignment between local and global flatness.}

	\label{figure 2}
\end{figure}
\begin{algorithm}[tb]
	\caption{\textbf{FedNSAM} Algorithm}
	\begin{algorithmic}[1]
		\STATE Input: $\beta, \lambda$, initial server model $\boldsymbol{\theta}^0$, number of clients $N$, number of communication rounds $T$, number of local iterations $K$, and local learning rate $\eta$.
		\STATE Initialize global momentum $\boldsymbol{m}^0=0$, and global model $\boldsymbol{\theta}^0$.
		\FOR{each round $t=1$ to $T$}
		\FOR{each selected client $i=1$ to $S$}
		\FOR{$k=0, \ldots, K$ local update}
		\STATE Update the local model $\boldsymbol{\theta}_i^t$ at Client $i$;
		\STATE $\boldsymbol{\theta}_{i, k+1/4}^t=\boldsymbol{\theta}_{i, k}^t+\lambda \boldsymbol{m}^t$; $  \triangleright$ Nesterov extrapolation
		\STATE $\boldsymbol{\delta}_{i, k}^t=\rho \frac{-\boldsymbol{m}^t}{\|\boldsymbol{m}^t\|}$; $  \triangleright$ Perturbation calculation
		\STATE $\boldsymbol{\theta}_{i, k+1/2}^t=\boldsymbol{\theta}_{i, k+1/4}^t+\boldsymbol{\delta}_{i, k}^t;$ $  \triangleright$ Perturbation model
		\STATE $\boldsymbol{\theta}_{i, k+1}^t\! =\! \boldsymbol{\theta}_{i, k}^t\!-\!\eta \nabla F_i\left(\boldsymbol{\theta}_{i, k+1/2}^t; \zeta_i\right)$; $\triangleright$ Updating
		\ENDFOR
        \STATE$\boldsymbol{\Delta}_i^t=\boldsymbol{\theta}_{i, K}^t-\boldsymbol{\theta}_{i, 0}^t$;
		\STATE Client $i$ sends $\boldsymbol{\Delta}_i^t$ back to the server.
		\ENDFOR
		\STATE Server averages models: 
		\STATE$\boldsymbol{\Delta}^t=\frac{1}{S}\sum_{i \in S_t} \boldsymbol{\Delta}_i^t $;\\
		\STATE$\boldsymbol{m}^t=\lambda \boldsymbol{m}^{t-1}+\boldsymbol{\Delta}^t$; $  \triangleright$ Nesterov momentum\\
		\STATE$\boldsymbol{\theta}^t=\boldsymbol{\theta}^{t-1}+\boldsymbol{m}^t$.
		\ENDFOR
	\end{algorithmic}
\end{algorithm}

\subsection{The proposed Algorithm}
Nesterov Accelerated Gradient (NAG) was proposed by Yurii Nesterov \cite{nesterov2013introductory} to accelerate the convergence of the gradient descent algorithm. The following is a core description of NAG algorithm:
\vspace{-2mm}
\begin{align}
& \boldsymbol{\theta}_{t+1 / 2}=\boldsymbol{\theta}_t-\lambda v_t, v_{t+1}=\lambda v_t+\eta \nabla F\left(\boldsymbol{\theta}_{t+1 / 2}\right), \\
& \boldsymbol{\theta}_{t+1}=\boldsymbol{\theta}_t+v_{t+1},
\end{align}
where $0\leq \lambda <1$ is the momentum parameter. The key difference between NAG and SAM  is the direction of the extra point ($\boldsymbol{\theta}_{t+1 / 2}$) search.  NAG searches for an extra point in the direction of $v_t$, while the SAM method searches for an extra point in the direction of gradient $\nabla F\left(\boldsymbol{\theta}_{t}\right)$. \textbf{FedNSAM} combines the flattening effect of SAM with the acceleration effect of extrapolation. The $i$-th client   updates local model $\boldsymbol{\theta}_i^t$ as follows:
\vspace{-2mm}
\begin{align}
& \boldsymbol{\theta}_{i, k+1 / 4}^t=\boldsymbol{\theta}_{i, k}^t+\lambda \boldsymbol{m}^t,  \boldsymbol{\delta}_{i, k}^t=\rho \frac{-\boldsymbol{m}^t}{\left\|\boldsymbol{m}^t\right\|},  \\
& \boldsymbol{\theta}_{i, k+1 / 2}^t=\boldsymbol{\theta}_{i, k+1 / 4}^t+\boldsymbol{\delta}_{i, k}^t,  \boldsymbol{\theta}_{i, k+1}^t =\boldsymbol{\theta}_{i, k}^t-\eta \nabla F_i\left(\boldsymbol{\theta}_{i, k+1 / 2}^t; \zeta_i\right).
\end{align}
Since directly calculating the global gradient $\nabla F\left(\boldsymbol{\theta}_t\right)$ is impractical in FL, we utilize the global momentum $\boldsymbol{m}_t$ as an approximation, $\boldsymbol{m}^{t}:=\lambda \boldsymbol{m}^{t-1}+\boldsymbol{\Delta}^t$.  Estimating the global gradient using an exponential moving average to compute $\boldsymbol{m}^{t}$ also reduces the effect of stochastic gradient noise \cite{li2024friendly}. The local updating of \textbf{FedNSAM} can be seen in Figure \ref{figure 2}  (b).

\subsection{Discussion}
\textbf{$\bullet$ }\textbf{Comparison with MoFedSAM and FedGAMMA: } FedMoSAM uses heavy ball momentum $\boldsymbol{\Delta}^t$  to accelerate the local model, $\boldsymbol{\Delta}^t=\frac{1}{S}\sum_{i \in S_t} \boldsymbol{\Delta}_i^t $, $\boldsymbol{\Delta}_i^t = \boldsymbol{\theta}_{i, K}^t-\boldsymbol{\theta}_{i, 0}^t$. local update as $\boldsymbol{\theta}_{i, k+1}^t=\boldsymbol{\theta}_{i, k}^t-\eta\big(\lambda g_{i, k+1 / 2}^t+\frac{1-\lambda}{\eta K} \boldsymbol{\Delta}^t\big)$, which is  similar to FedCM. However, FedMoSAM did not have an accelerating effect in theory, while our \textbf{FedNSAM} 
 algorithm uses $\boldsymbol{m}_t$ as Nesterve momentum. Furthermore, when the client participation rate is low, $\boldsymbol{\Delta}^t$ does not accurately estimate the amount of global change. Instead, $\boldsymbol{m}^t=\lambda \boldsymbol{m}^{t-1}+\boldsymbol{\Delta}^t$, exponential moving average of the historical gradient, giving higher weight to the latest global change, which allows accurate estimation of the global updating. FedGAMMA introduces the variance reduction control variable used in SCAFFOLD \cite{karimireddy2020scaffold} to control the consistency of the global and client models. Both algorithms only consider correction to $\nabla F(\boldsymbol{\theta}_{k+1/2}^t)$, and do not consider correction to $\boldsymbol{\delta}_{i, k}^t$ in Table \ref{table 1}.\\
\textbf{$\bullet$ }\textbf{Comparison with FedSMOO and FedLESAM: }FedLESAM uses $\boldsymbol{\Delta}^{t}$ as the global perturbation direction estimate, and the \textbf{FedNSAM} algorithm uses $\boldsymbol{m}^{t}$ as the global perturbation direction estimate. The global perturbation estimate for FedLESAM, $\boldsymbol{\Delta}^t=\frac{1}{S}\sum_{i \in S_t} \boldsymbol{\Delta}_i^t $, $\boldsymbol{\Delta}_i^t = \boldsymbol{\theta}_{i, K}^t-\boldsymbol{\Delta}_{i, 0}^t$, which is the difference between the previous active participation round and the current round before the global model is received. When client participation is low, $\boldsymbol{\Delta}^t$ does not accurately estimate the global perturbation. Instead, $\boldsymbol{m}^t=\lambda \boldsymbol{m}^{t-1}+\boldsymbol{\Delta}^t$, which allows a more accurate estimation of the global perturbation. The FedSMOO  uses  $\nabla F_i(\boldsymbol{\theta}_{i, k}^t)-\mu_i-s$ as the global perturbation direction estimate. However, its computational cost and communication are twice as much as that of our algorithm. However, both algorithms only consider the correction to $\boldsymbol{\delta}_{i, k}^t$, and do not consider the correction to $\nabla F(\boldsymbol{\theta}_{k+1/2}^t)$  in Table \ref{table 1}.\\

\section{Theoretical Analysis}
We analyze generalization based on  following assumptions: \\
\textbf{Assumption 1} \textit{(Smoothness).  $F_i$ is $L$-smooth for all $i \in$ $[N]$, 
\begin{align}
\left\|\nabla F_i(\boldsymbol{\theta}_{1})-\nabla F_i(\boldsymbol{\theta}_{2})\right\| \leq L\|\boldsymbol{\theta}_{1}-\boldsymbol{\theta}_{2}\|,
\end{align}
for all $\boldsymbol{\theta}_{1}, \boldsymbol{\theta}_{2}$ in its domain and $i \in[N]$.}\\
\textbf{Assumption 2} \textit{(Bounded variance of data heterogeneity). The global variability of the local gradient of the loss function is bounded by $\sigma_g^2$ for all $i \in[N]$, 
\vspace{-2mm}
\begin{align}
\left\|\nabla F_i\left(\boldsymbol{\theta}\right)-\nabla F\left(\boldsymbol{\theta}\right)\right\|^2 \leq \sigma_g^2.
\end{align}}\\
\textbf{Assumption 3}\textit{ (Bounded variance of stochastic gradient). The stochastic gradient $\nabla F_i\left(\boldsymbol{\theta}, \xi_i\right)$, computed by the $i$-th client of model parameter $\boldsymbol{\theta}$ using mini-batch $\xi_i$ is an unbiased estimator $\nabla F_i(\boldsymbol{\theta})$ with variance bounded by $\sigma^2$, i.e.,
\begin{align}
\mathbb{E}_{\xi_i}\left\|\nabla F_i\left(\boldsymbol{\theta}, \xi_i\right)-\nabla F_i(\boldsymbol{\theta})\right\|^2 \leq \sigma^2,
\end{align}
$\forall i \in[N]$, where the expectation is over all local datasets.}\\
\subsection{Convergence Results}

\begin{theorem}[Convergence for non-convex functions] \label{theorem 1}
	Suppose that local  $\left\{F_i\right\}_{i=1}^N$ are non-convex and $L$-smooth. By setting $\eta \leq \frac{(1-\lambda)^2}{128 K L}$, $\rho=\sqrt{\frac{1}{T}}$, \textbf{FedNSAM} satisfies
    \vspace{-1mm}
	$$
	\begin{aligned}
		& \frac{1}{T} \sum_{t=1}^T \mathbb{E}\left[\left\|\nabla F\left(\boldsymbol{\theta}^{t\!-\!1}\!+\!\lambda \boldsymbol{m}^{t-1}\right)\right\|\right]   \leq\\
		&  \mathcal{O}\Big(\frac{M_1 \sqrt{L F}}{\sqrt{T KS(1\!-\!\lambda)}}\!+\!\frac{(L F)^{\frac{2}{3}} M_2^{\frac{1}{3}}}{T^{\frac{2}{3}}}\!+\!\frac{ L F}{T(1\!-\!\lambda)}\!+\!\frac{ L^4(1-\lambda)}{T} \Big),
	\end{aligned}
	$$
	where $M_1^2:=\sigma^2+K\big(1-\frac{S}{N}\big) \sigma_g^2, M_2:=\frac{\sigma^2}{K}+\sigma_g^2$, and $F:=F\big(\boldsymbol{\theta}^0\big)-F\big(\boldsymbol{\theta}^{\star}\big)$, $\left|S_t\right|=S$.
\end{theorem}  
Theorem \ref{theorem 1} provides the non-convex convergence rate of \textbf{FedNSAM}, which matches the best convergence rate of existing FL methods \cite{karimireddy2020scaffold}. We provide the convergence of \textbf{FedNSAM} under partial client participation. The convergence rate of \textbf{FedNSAM} is $\mathcal{O}\Big(\frac{M_1 \sqrt{L F}}{\sqrt{T KS(1-\lambda)}}+\frac{(L F)^{\frac{2}{3}} M_2^{\frac{1}{3}}}{T^{\frac{2}{3}}}+\frac{ L F}{T(1-\lambda)}+\frac{ L^4(1-\lambda)}{T} \Big)$,
which is better than the convergence rate of FedSAM's  $\mathcal{O}\big(\frac{L F}{\sqrt{T K S}}+\frac{\sqrt{K}\sigma_g^2}{\sqrt{T S}}+\frac{L^2 \sigma^2}{T^{3 / 2} K}+\frac{L^2}{T^2}\big)$ \cite{qu2022generalized}. That is beacause $\mathcal{O}\big(\frac{M_1 \sqrt{L F}}{\sqrt{T KS(1-\lambda)}} \big)<\mathcal{O}\big(\frac{L F}{\sqrt{T K S}}\big)$ , and $ \sqrt{L F}<L F$. Moreover, the generalization result of FedSAM is commonly suitable for our \textbf{FedNSAM}.
\subsection{\textbf{Flatness Distance } Analysis}

\begin{table*}[tb]
	\vspace{-2mm}
	\centering
	\setlength{\tabcolsep}{1pt}
	\caption{Comparison of testing accuracy (\%) on CIFAR10 and CIFAR100,  $E=5$, Dirichlet-0.6 respectively. }
\setlength{\belowcaptionskip}{1pt} 
	\vspace{-3mm}
	\begin{tabular}{llccccccccccccc}
		\midrule[1.5pt]
		& \multicolumn{6}{c}{CIFAR10} & \multicolumn{6}{c}{CIFAR100} \\
		\cmidrule(lr){2-7} \cmidrule(lr){8-13} 
		Method & \multicolumn{2}{c}{LeNet-5} & \multicolumn{2}{c}{VGG-11} & \multicolumn{2}{c}{ResNet-18}& \multicolumn{2}{c}{LeNet-5} & \multicolumn{2}{c}{VGG-11} & \multicolumn{2}{c}{ResNet-18} \\
		\cmidrule(lr){2-3} \cmidrule(lr){4-5} \cmidrule(lr){6-7}	\cmidrule(lr){8-9} \cmidrule(lr){10-11} \cmidrule(lr){12-13}
		& Acc.(\%) & Rounds & Acc.(\%)& Rounds& Acc.(\%) & Rounds  & Acc.(\%)  & Rounds& Acc.(\%)  & Rounds& Acc.(\%)  & Rounds \\
		& 1000R & 78\%  &1000R & 80\% &1000R & 85\%& 1000R & 52\%  &1000R & 53\% &1000R & 55\% \\
		\midrule
		FedAvg  & $79.63\pm0.27$& 574  &$84.14\pm0.32$&313  & $  88.92\pm0.24$&542 &$41.15\pm0.18$&1000+   &$48.94\pm0.39$&1000+ & $54.25\pm0.39$ &1000+  \\
		FedAvgM     & $81.15\pm0.23$ &659 & $83.74\pm0.11$&794  & $89.05\pm0.36$&656 &$48.35\pm0.29$&1000+   &$51.94\pm0.17$&1000+  &$60.91\pm0.37$&516    \\
		SCAFFOLD   & $79.43\pm0.33$&799  & $86.94\pm0.09$&272  &  $88.56\pm0.21$& 543 & $50.21\pm0.35$&1000+   &$50.81\pm0.28$ &1000+  &  $54.12\pm0.15$ &1000+\\
		FedACG & $ 82.84\pm0.28$& 301  & $84.86\pm0.14$ &249  &$90.62\pm0.31$ &\textbf{222 }& $53.04\pm0.19$&729  & $ 51.55\pm0.34$&1000+&$61.92\pm0.31$ &521\\
		FedSAM       &$ 81.21\pm0.16$&361  &$85.26\pm0.27$&227   &$86.78\pm0.13$ &313 & $48.12\pm0.22$&1000+  & $47.71\pm0.09$ &1000+ &   $47.83\pm0.25$&628 \\
		MoFedSAM & $83.42\pm0.20$&255  & $- $&1000+ &$86.66\pm0.37$&530 & $ 50.11\pm0.10$&1000+  & $-  $&1000+&$60.15\pm0.24$&1000+\\
		FedGAMMA   & $ 81.02\pm0.24$&371& $86.33\pm0.25$&169   &  $88.35\pm0.29$& 542 & $41.33\pm0.23$  &1000+& $53.42\pm0.32$&908& $49.83\pm0.35$&1000+\\
		FedLESAM  & $79.75\pm0.30$&496  & $83.54\pm0.18$&169   &  $89.01\pm0.35$&480 & $41.33\pm0.24$  &1000+& $49.26\pm0.21$&1000+& $52.11\pm0.21$&1000+\\
        \rowcolor{LightRed}
		\textbf{\textbf{FedNSAM}}   & $\mathbf{83.82} \pm0.15$& \textbf{181} &$\mathbf{87.86} \pm0.29$ &\textbf{137}  &$ \mathbf{91.35} \pm0.33$ &272  & $\mathbf{53.61} \pm0.20$&\textbf{633} &$\mathbf{56.33} \pm0.31$ &\textbf{334} & $\mathbf{66.04} \pm0.11$&\textbf{316 }\\
		\midrule[1.5pt]
	\end{tabular}
	
	\label{table 2}
\end{table*}

\begin{theorem}\label{theorem 2}
For FedSAM, if we choose $\eta=\mathcal{O}\left(\frac{1}{\sqrt{T} K L}\right)$ and $\rho=\mathcal{O}\left(\frac{1}{\sqrt{T}}\right),  \boldsymbol{\Delta}_{\mathcal{D}}$ is then bounded as follows:
\vspace{-2mm}
\begin{align}
\boldsymbol{\Delta}_{\mathcal{D}}\!\leq\! \mathcal{O}\left(\frac{\sigma^2}{K T^2} \!+\!\frac{\sigma_g^2}{L^2 T}\! +\!\frac{\sqrt{K} \sigma_g^2}{L^2  \sqrt{N}T^{3 / 2}}\!+\!\frac{L^2}{T}\!+\!\frac{F}{L \sqrt{K N} T^{3 / 2}}\right)
\end{align}
\end{theorem}

\begin{theorem}\label{theorem 3} For \textbf{FedNSAM}, if we choose $\eta=\mathcal{O}\left(\frac{(1-\lambda)^2}{ K L}\right)$ and $\rho=\mathcal{O}\left(\frac{1}{\sqrt{T}}\right),  \boldsymbol{\Delta}_{\mathcal{D}}$ is then bounded as follows:
\vspace{-2mm}
\begin{align}
\boldsymbol{\Delta}_{\mathcal{D}} \!\leq\! \mathcal{O}\left(\frac{\sigma^2}{K T^2} \!+\!\frac{\sigma_g^2}{L^2 T} \!+\!\frac{L^2}{T}\!+\!\frac{\sigma \sqrt{F}}{T^{3 / 2} L^{3 / 2} \sqrt{K N(1\!-\!\lambda)}}\right)
\end{align}
\end{theorem}

As shown in Theorem \ref{theorem 2} and \ref{theorem 3}, the term $\sigma_g^2$, which increases when heterogeneity gets worse, directly determines the upper bound of $\boldsymbol{\Delta}_{\mathcal{D}}$.  The local model $\boldsymbol{\theta}_i$ largely deviates from the global model $\boldsymbol{\theta}$ with high data heterogeneity.
For \textbf{FedNSAM}, the upper bound of $\boldsymbol{\Delta}_{\mathcal{D}}$ formalized by Theorem \ref{theorem 3} is better than FedSAM (Theorem \ref{theorem 2}).

\section{Experiments}

\subsection{Experimental Settings}

\textbf{{Datasets:}} We evaluate our algorithms on the CIFAR10 \cite{krizhevsky2009learning}, CIFAR100 \cite{krizhevsky2009learning}, Tiny ImageNet \cite{le2015tiny}. For non-IID data setup, we simulate the data heterogeneity by sampling the label ratios from a Dirichlet distribution \cite{hsu2019measuring}.\\ 
\textbf{{Models:}} To test the robustness of our algorithms, we use standard classifiers (including LeNet-5 \cite{lecun2015lenet}, VGG-11 \cite{simonyan2014very}, and ResNet-18 \cite{he2016deep}), Vision Transformer (ViT-Base) \cite{dosovitskiy2020image}, Swin transformer (Swin-Small, Swin-Base) \cite{liu2021swin}. We first verify the effectiveness
of our algorithm by training the model using convolutional networks on the CIFAR10 \cite{krizhevsky2009learning}, CIFAR100 \cite{krizhevsky2009learning}, Tiny ImageNet \cite{le2015tiny} datasets. 
\\
\textbf{{Methods:}} We compare \textbf{FedNSAM} with many 
FL baselines, including FedAvg \cite{mcmahan2017communication}, FedAvgM \cite{hsu2019measuring}, SCAFFOLD \cite{karimireddy2020scaffold}, FedACG \cite{kim2024communication},  FedSAM \cite{qu2022generalized}, MoFedSAM \cite{qu2022generalized}, FedGAMMA \cite{dai2023fedgamma}, FedLESAM \cite{fan2024locally} .\\
\textbf{Hyper-parameter Settings:}
The number of clients is 100. batch size $B\!=\!50$, local epoch $E \!=\! 5$, $K \!=\! 50$. We set the grid search range of the client learning rate by $\eta \!\in\!\{10^{-3}, 3\times 10^{-3},...,10^{-1}, 3\times 10^{-1}\}$. Learning rate decay per round is $0.998$, total $T=1,000$. 
Specifically, we set $\rho\!=\!0.  1$, $\lambda=0.85$ for \textbf{FedNSAM}.

\begin{figure}[tb]
	\centering
	\begin{minipage}[b]{0.235\textwidth}
		\centering
		\subcaptionbox{Resnet-18, CIFAR10}{\includegraphics[width=\textwidth]{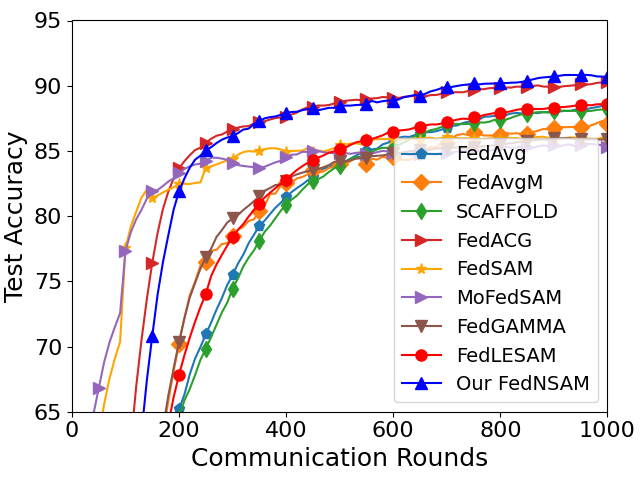}}
	\end{minipage}
	\begin{minipage}[b]{0.235\textwidth}
		\centering
		\subcaptionbox{ResNet-18, CIFAR100}{\includegraphics[width=\textwidth]{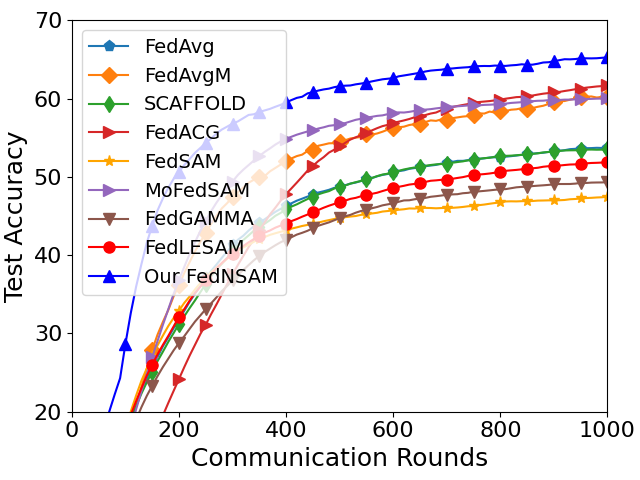}}
	\end{minipage}
	\caption{Convergence plots for \textbf{FedNSAM} and other baselines  on Dirichlet-0.6 of CIFAR10 and CIFAR100 with ResNet-18.}
	\label{figure 3}
\end{figure}

\subsection{Results on Convolutional Neural Networks}
We first conduct experiments on convolutional neural networks to verify the superiority of the \textbf{FedNSAM} algorithm on CIFAR100 and CIFAR10 datasets with Dirichlet-0.6 data split on 100 clients, $10 \%$ participating setting and $E=5$, batch size $B\!=\!50$. We selected three benchmark convolutional neural network frameworks, which are LeNet-5 \cite{lecun2015lenet}, VGG-11 \cite{simonyan2014very}, and ResNet-18 \cite{he2016deep} in Table \ref{table 2}.

From Table \ref{table 2}  (Figure \ref{figure 3}), we have the following observations:
(i)  The \textbf{FedNSAM} algorithm achieves the highest accuracy with different datasets and network architectures with the training rounds 1000. \textbf{FedNSAM} achieves the same accuracy with fewer training rounds. For example, in the case of ResNet-18 training the CIFAR100 dataset for 1000 rounds, \textbf{FedNSAM}'s accuracy of 66.04\% is 12.21\% higher than FedSAM's 47.83\%. \textbf{FedNSAM} is also higher than other variants of FedSAM. Reaching 55\% accuracy, \textbf{FedNSAM} uses only 316 rounds, accelerating FedSAM by more than 3 $\times$ and accelerating FedACG 1.6 $\times$. 
(ii) Compared to the latest momentum FL algorithms (FedAvgM, FedACG) and a series of variants of the SAM algorithm (MoFedSAM, FedGAMMA, FedLESAM), the \textbf{FedNSAM} algorithm has the best convergence speed and final generalization.

\begin{figure}[tb]
	\centering
	\begin{minipage}[b]{0.235\textwidth}
		\centering
		\subcaptionbox{ Impact of  $\lambda$}{\includegraphics[width=\textwidth]{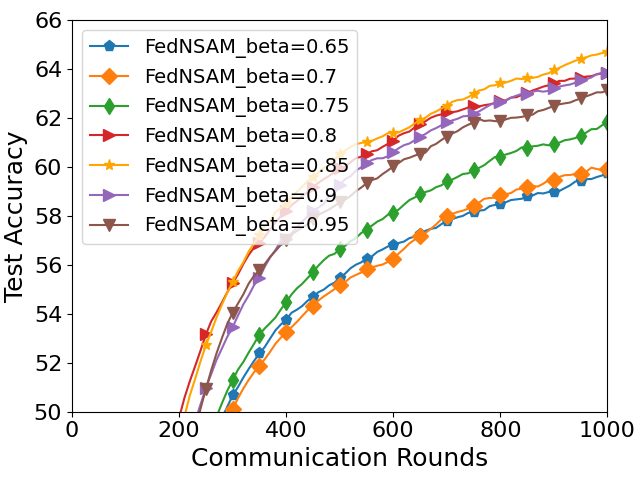}}
	\end{minipage}
	\begin{minipage}[b]{0.235\textwidth}
		\centering
		\subcaptionbox{Impact of $\rho$}{\includegraphics[width=\textwidth]{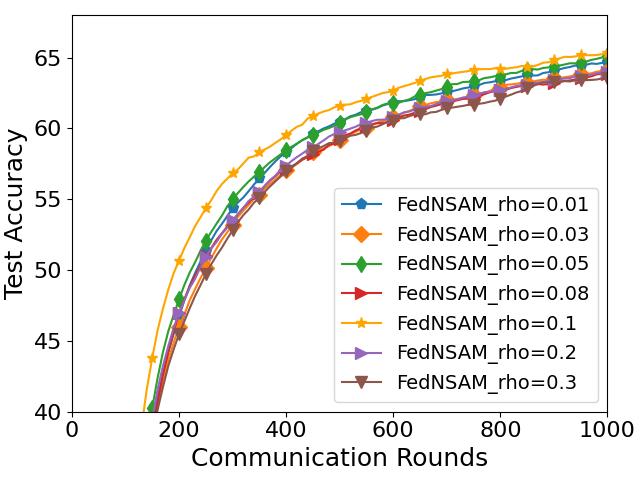}}
	\end{minipage}
	\caption{Convergence plots for \textbf{FedNSAM} with different $\lambda$ and  $\rho$, CIFAR100 datasets with ResNet-18.}
	\label{figure 4}
\end{figure}
\textbf{Impact of $\lambda$: }As shown in our theoretical analysis, $\lambda$ reduces the effect of client heterogeneity in local updating. To validate this, we also perform experiments to analyze the effect of $\lambda$, the only algorithm-dependent hyperparameter of \textbf{FedNSAM}, on the convergence and performance of \textbf{FedNSAM} algorithms.
We test \textbf{FedNSAM} with $\lambda$ taken values in $\{0.95,0.9,0.85,0.8,0.75,0.7,0.65\}$, training ResNet-18 on CIFAR100 datasets with Dirichlet-0.6 split on 100 clients, $10 \%$ participating setting. The test accuracies and the convergence plots are provided in Figure \ref{figure 4}. We find that \textbf{FedNSAM} successfully converges to stationary points under all these $\lambda$ choices, as guaranteed by our convergence analysis. However, the stationary points of different $\lambda$ show different generalization abilities, which results in varying test accuracies in Figure \ref{figure 4}. We note that setting $\lambda$ too small or too large will harm the convergence and generalization of \textbf{FedNSAM}. As shown in  Figure \ref{figure 4},  as the $\lambda$ ($\lambda<0.95$) increases, the acceleration effect of the \textbf{FedNSAM} algorithm becomes more obvious.  Empirically, we find that performance is best when setting $\lambda$ to about 0.85, which aligns with traditional momentum algorithms FedACG and FedAvgM.

\textbf{Impact of $\rho$:}  
The hyperparameter $\rho$ controls the perturbation radius in sharpness-aware updates, influencing the model's ability to seek flat minima. A small $\rho$ may underexplore curvature, while a large $\rho$ can cause instability or underfitting, especially with high data heterogeneity. Empirically, we find that $\rho=0.1$ offers the best balance between generalization and stability across datasets and models. As shown in Figure~\ref{figure 4} (right), increasing $\rho$ to 0.1 improves test accuracy, while larger values degrade performance. Thus, we set $\rho=0.1$ as the default in all experiments.

\begin{table}[tb]	
	\centering
	\setlength{\tabcolsep}{0pt}
	\caption{Comparision the accuracy of different Transformer models with Dirichlet-0.1, 5\% participation.}
\setlength{\belowcaptionskip}{2pt} 
	\begin{tabular}{p{1.8cm}lccccccc}
		\midrule[1.5pt]
		Data Set&\multicolumn{6}{c}{Tiny ImageNet}\\ 
		\cmidrule(lr){1-7} 
		Model & \multicolumn{2}{c}{Swin-Small} & \multicolumn{2}{c}{Swin-Base} & \multicolumn{2}{c}{ ViT-Base} \\
		Param  &\multicolumn{2}{c}{50M}& \multicolumn{2}{c}{80M}  & \multicolumn{2}{c}{88M}  \\
		FLOPs  &\multicolumn{2}{c}{8.7G}& \multicolumn{2}{c}{15.4G}  & \multicolumn{2}{c}{55.4G} \\
		\cmidrule(lr){2-3} \cmidrule(lr){4-5} \cmidrule(lr){6-7}
		& Acc & Rounds & Acc& Rounds& Acc& Rounds\\
		Method& 100R & 68\%  &100R & 68\% &100R & 68\%\\
        \midrule
		FedAvg & 69.80$\pm0.21$  &100+& 67.41&100+ & 64.64&100+ \\
		FedAvgM &67.63 &100+   & 68.18&75  &   69.77&80    \\
		SCAFFOLD &69.62  &82&67.37&88 &65.02&100+ \\
		FedAGC &64.38  &100+ &65.54 &100+ &70.22&41 \\
		FedSAM  &63.64 &100+ &64.56&100+ &65.46&100+  \\
		MoFedSAM &64.26 &100+ &64.45 &100+ &64.60&100+  \\
		FedGAMMA &65.02 &100+    &65.64&100+  &66.02&100+ \\
		FedLESAM &65.11&100+   &65.89&100+ &66.28&100+\\
        \rowcolor{LightRed}
		\textbf{\textbf{FedNSAM}} & \textbf{70.12} &\textbf{67} &\textbf{70.86} &\textbf{42} &\textbf{71.23} &\textbf{26} \\
		\midrule[1.5pt]
	\end{tabular}
	\label{table 3}
\end{table}

\begin{table*}[t]
\centering
\caption{Comparison of testing accuracy (\%) and convergence rounds on CIFAR100 ($E=5$) for different participation rates (left) and different data heterogeneity levels (right).}
\setlength{\belowcaptionskip}{1pt} 
\setlength{\tabcolsep}{1.5pt}
\begin{tabular}{lcccccc|cccccc}
\toprule
& \multicolumn{6}{c}{\textbf{Participation Rate}} & \multicolumn{6}{c}{\textbf{Data Heterogeneity (Dirichlet-$\alpha$)}} \\
\cmidrule(lr){2-7} \cmidrule(lr){8-13}
Method & \multicolumn{2}{c}{$p=2\%$} & \multicolumn{2}{c}{$p=5\%$} & \multicolumn{2}{c}{$p=10\%$} 
& \multicolumn{2}{c}{$\alpha=0.1$} & \multicolumn{2}{c}{$\alpha=0.3$} & \multicolumn{2}{c}{$\alpha=0.6$} \\
\cmidrule(lr){2-3} \cmidrule(lr){4-5} \cmidrule(lr){6-7} \cmidrule(lr){8-9} \cmidrule(lr){10-11} \cmidrule(lr){12-13}
& Acc & R& Acc & R & Acc & R
& Acc & R & Acc & R & Acc & R \\
\midrule
FedAvg & 51.61$\pm$0.35 & 1000+ & 51.53$\pm$0.12 & 1000+ & 54.25$\pm$0.39 & 1000+ 
& 45.81$\pm$0.19 & 1000+ & 52.53$\pm$0.15 & 1000+ & 54.25$\pm$0.39 & 1000+ \\

FedAvgM & 15.72$\pm$0.23 & 1000+ & 46.63$\pm$0.33 & 1000+ & 60.91$\pm$0.37 & 516 
& 48.63$\pm$0.40 & 1000+ & 58.84$\pm$0.08 & 808 & 60.91$\pm$0.37 & 516 \\

SCAFFOLD & 51.93$\pm$0.29 & 1000+ & 53.64$\pm$0.37 & 1000+ & 54.12$\pm$0.28 & 1000+ 
& 49.02$\pm$0.33 & 1000+ & 51.92$\pm$0.25 & 1000+ & 54.12$\pm$0.28 & 1000+ \\

FedACG & 54.01$\pm$0.16 & 1000+ & 61.94$\pm$0.26 & 506 & 61.92$\pm$0.31 & 521
& 53.56$\pm$0.34 & 1000+ & 60.87$\pm$0.12 & 425 & 61.92$\pm$0.31 & 521 \\

FedSAM & 49.02$\pm$0.18 & 1000+ & 48.53$\pm$0.24 & 1000+ & 47.83$\pm$0.09 & 1000+ 
& 40.18$\pm$0.27 & 1000+ & 46.02$\pm$0.22 & 1000+ & 47.83$\pm$0.09 & 1000+ \\

MoFedSAM & 44.22$\pm$0.09 & 1000+ & 60.64$\pm$0.34 & 511 & 60.15$\pm$0.24 & 616 
& 51.59$\pm$0.16 & 1000+ & 58.21$\pm$0.29 & 756 & 60.15$\pm$0.24 & 616 \\

FedGAMMA & 49.42$\pm$0.39 & 1000+ & 51.22$\pm$0.17 & 1000+ & 49.83$\pm$0.35 & 1000+ 
& 47.73$\pm$0.18 & 1000+ & 46.65$\pm$0.38 & 1000+ & 49.83$\pm$0.35 & 1000+ \\

FedLESAM & 54.21$\pm$0.31 & 1000+ & 54.22$\pm$0.28 & 1000+ & 52.11$\pm$0.21 & 1000+ 
& 48.74$\pm$0.30 & 1000+ & 53.34$\pm$0.17 & 1000+ & 52.11$\pm$0.21 & 1000+ \\

\rowcolor{LightRed}
\textbf{FedNSAM} & \textbf{56.92$\pm$0.27} & \textbf{872} & \textbf{62.21$\pm$0.11} & \textbf{496} & \textbf{66.04$\pm$0.11} & \textbf{316} 
& \textbf{58.53$\pm$0.26} & \textbf{695} & \textbf{63.65$\pm$0.36} & \textbf{372} & \textbf{66.04$\pm$0.11} & \textbf{316} \\
\bottomrule
\end{tabular}
\label{tab:merged}
\end{table*}

\begin{figure*}[t]
    \centering
    \begin{minipage}[b]{0.235\textwidth}
        \centering
        \subcaptionbox{Dirichlet-0.3}{\includegraphics[width=\textwidth]{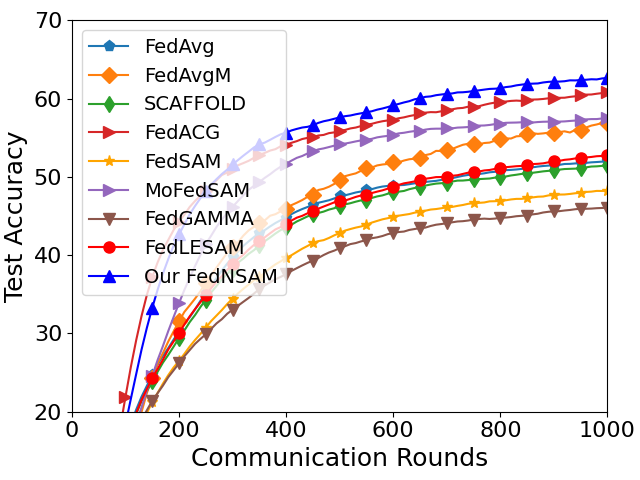}}
    \end{minipage}
    \begin{minipage}[b]{0.235\textwidth}
        \centering
        \subcaptionbox{Dirichlet-0.1}{\includegraphics[width=\textwidth]{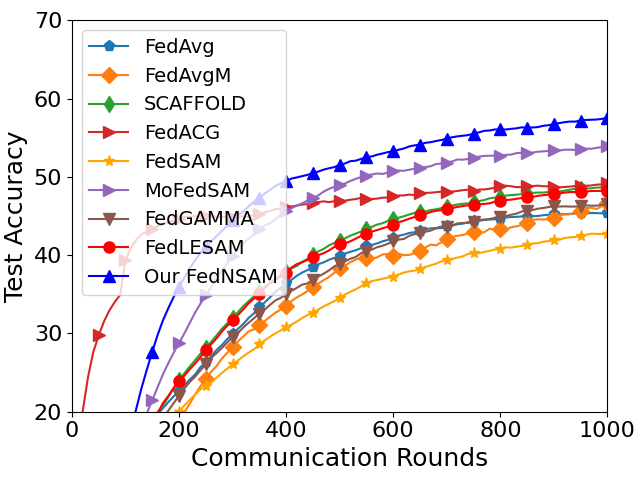}}
    \end{minipage}
    \begin{minipage}[b]{0.235\textwidth}
        \centering
        \subcaptionbox{$p = 2\%$}{\includegraphics[width=\textwidth]{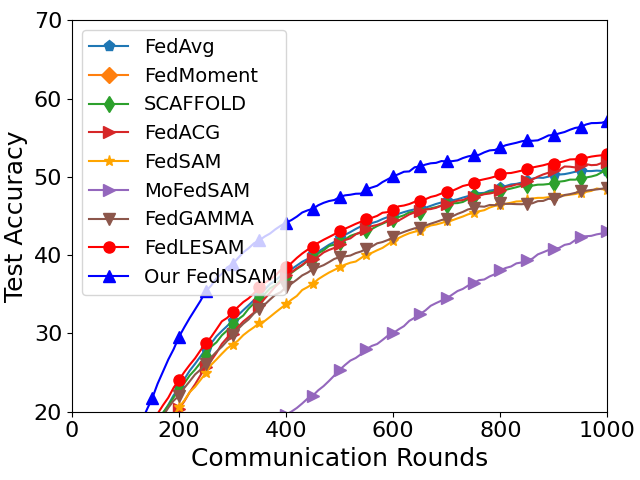}}
    \end{minipage}
    \begin{minipage}[b]{0.235\textwidth}
        \centering
        \subcaptionbox{$p = 5\%$}{\includegraphics[width=\textwidth]{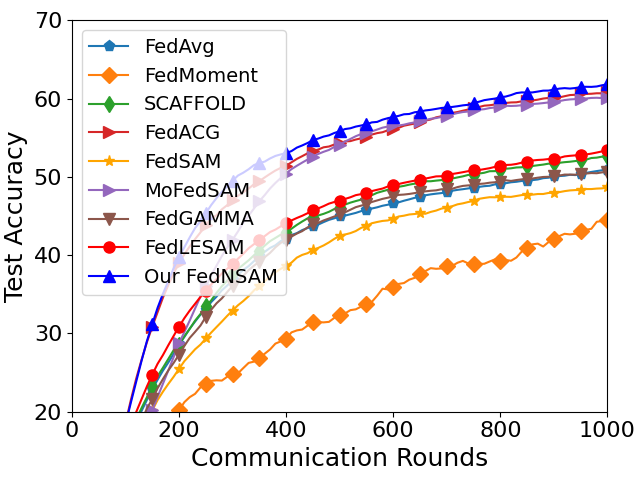}}
    \end{minipage}
    \setlength{\abovecaptionskip}{2pt} 
    \caption{Convergence plots of \textbf{FedNSAM} and baselines on CIFAR100 with ResNet-18 under different heterogeneity (a,b) and participation rates (c,d).}
    \label{fig:convergence_all}
\end{figure*}
    

\subsection{Results on Vision Transformer Models}
To evaluate the effectiveness of our method on large-scale models, we conduct experiments on the Tiny ImageNet dataset using Swin-Small, Swin-Base, and ViT-Base, all initialized with ImageNet-22k pre-trained weights. We train for 100 communication rounds under Dirichlet-0.1 with a learning rate of 0.01 (decayed by 0.99 per round) and a batch size of 16.
Table~\ref{table 3} reports the parameter sizes and theoretical FLOPs of each model (input size 224), along with the performance comparison. \textbf{FedNSAM} consistently achieves the highest accuracy—70.12\%, 70.86\%, and 71.23\%—on Swin-Small, Swin-Base, and ViT-Base respectively, while requiring significantly fewer rounds (67, 42, and 26). These results clearly demonstrate the superior generalization and training efficiency of \textbf{FedNSAM}, especially for large-scale transformer models and high-complexity datasets.

\begin{figure*}[tb]
	\centering
	\begin{minipage}[b]{0.16\textwidth}
		\centering
		\subcaptionbox{FedSAM}{\includegraphics[width=\textwidth]{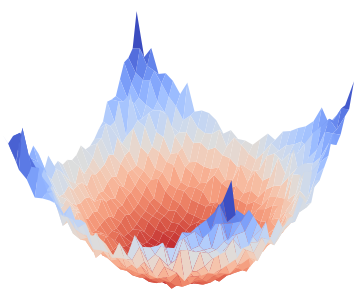}}
	\end{minipage}
	\hfill
	\begin{minipage}[b]{0.16\textwidth}
		\centering
		\subcaptionbox{MoFedSAM}{\includegraphics[width=\textwidth]{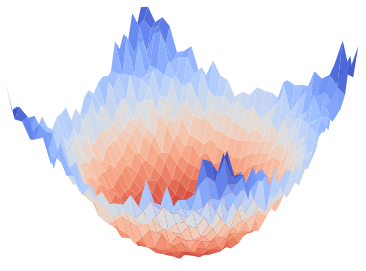}}
	\end{minipage}
	\hfill
	\begin{minipage}[b]{0.16\textwidth}
		\centering
		\subcaptionbox{FedGAMMA}{\includegraphics[width=\textwidth]{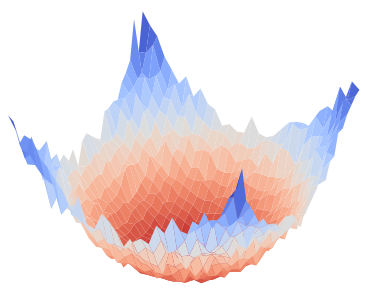}}
	\end{minipage}
	\hfill
	\begin{minipage}[b]{0.16\textwidth}
		\centering
		\subcaptionbox{FedLESAM}{\includegraphics[width=\textwidth]{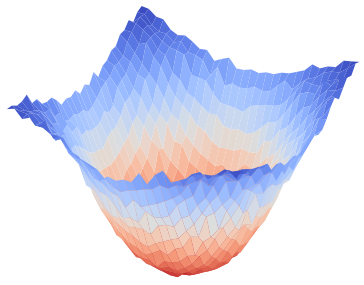}}
	\end{minipage}
	\hfill
	\begin{minipage}[b]{0.16\textwidth}
		\centering
        \subcaptionbox{FedAvg}{\includegraphics[width=\textwidth]{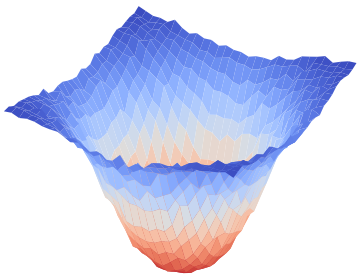}}
	\end{minipage}
	\hfill
	\begin{minipage}[b]{0.16\textwidth}
		\centering
        \subcaptionbox{\textbf{FedNSAM}  (ours)}{\includegraphics[width=\textwidth]{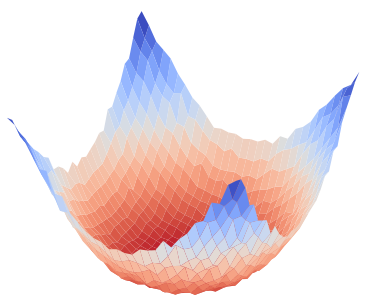}}
	\end{minipage}
	\begin{minipage}[b]{0.16\textwidth}
		\centering
		\subcaptionbox{FedSAM}{\includegraphics[width=\textwidth]{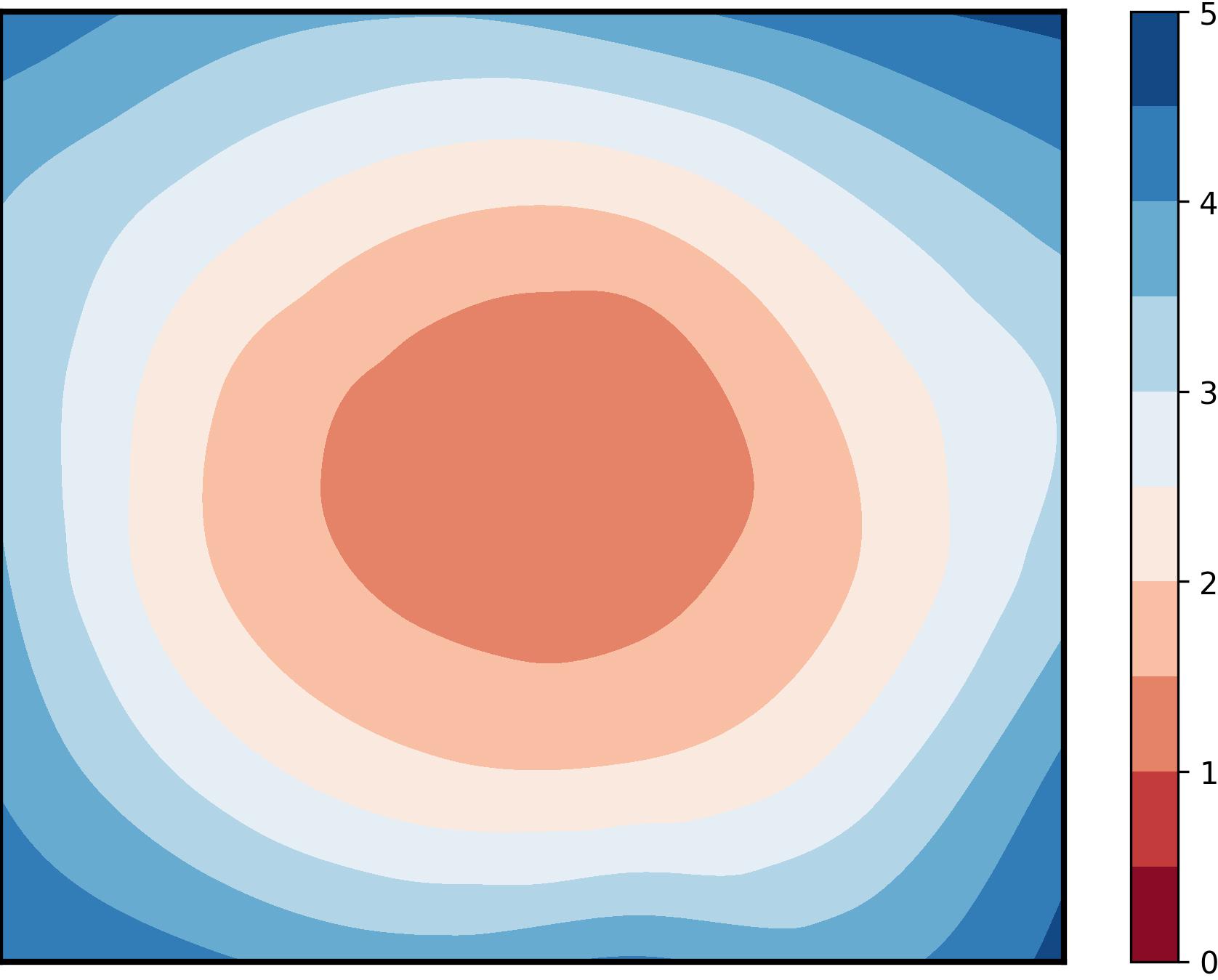}}
	\end{minipage}
	\hfill
	\begin{minipage}[b]{0.16\textwidth}
		\centering
		\subcaptionbox{MoFedSAM}{\includegraphics[width=\textwidth]{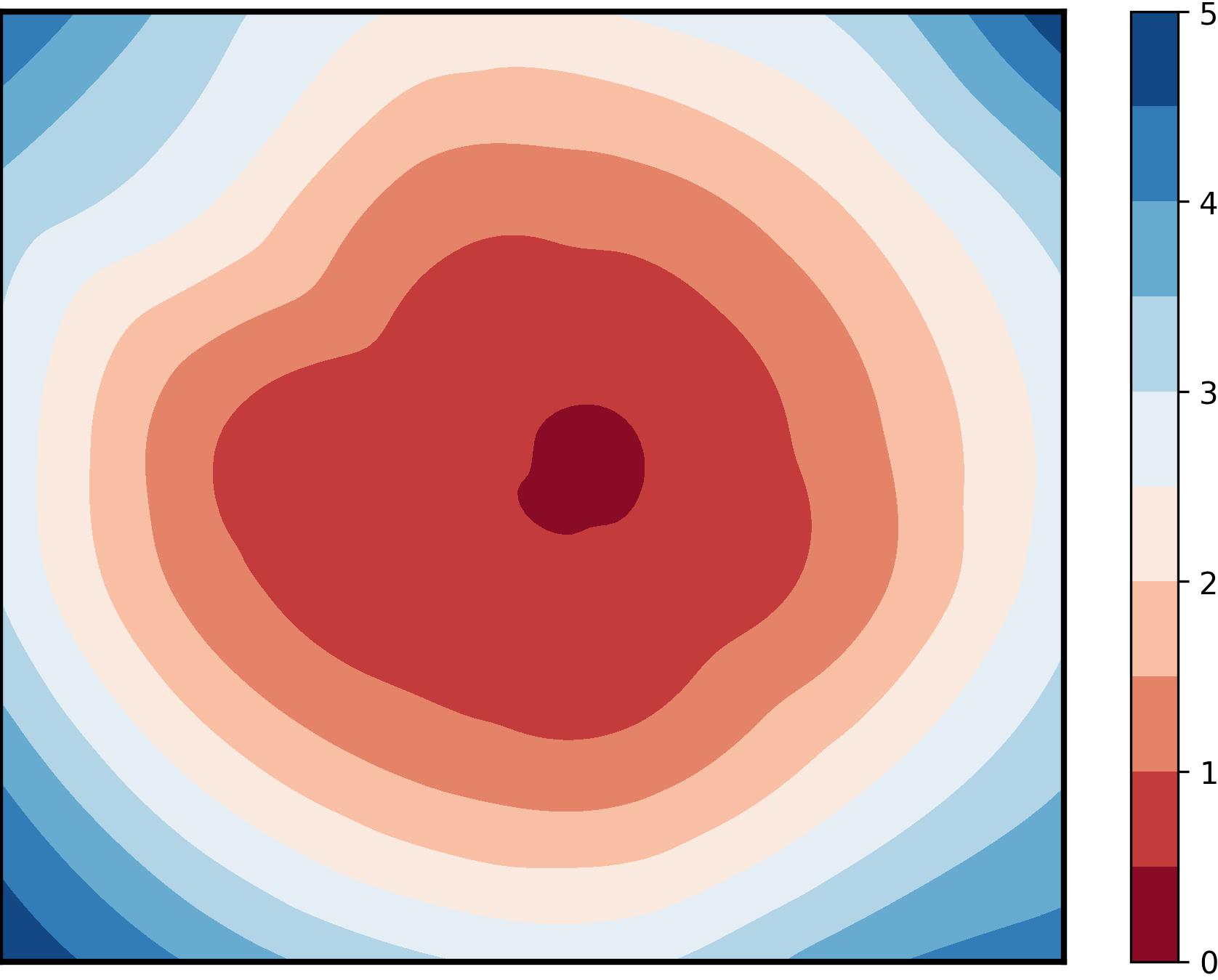}}
	\end{minipage}
	\hfill
	\begin{minipage}[b]{0.16\textwidth}
		\centering
		\subcaptionbox{FedGAMMA}{\includegraphics[width=\textwidth]{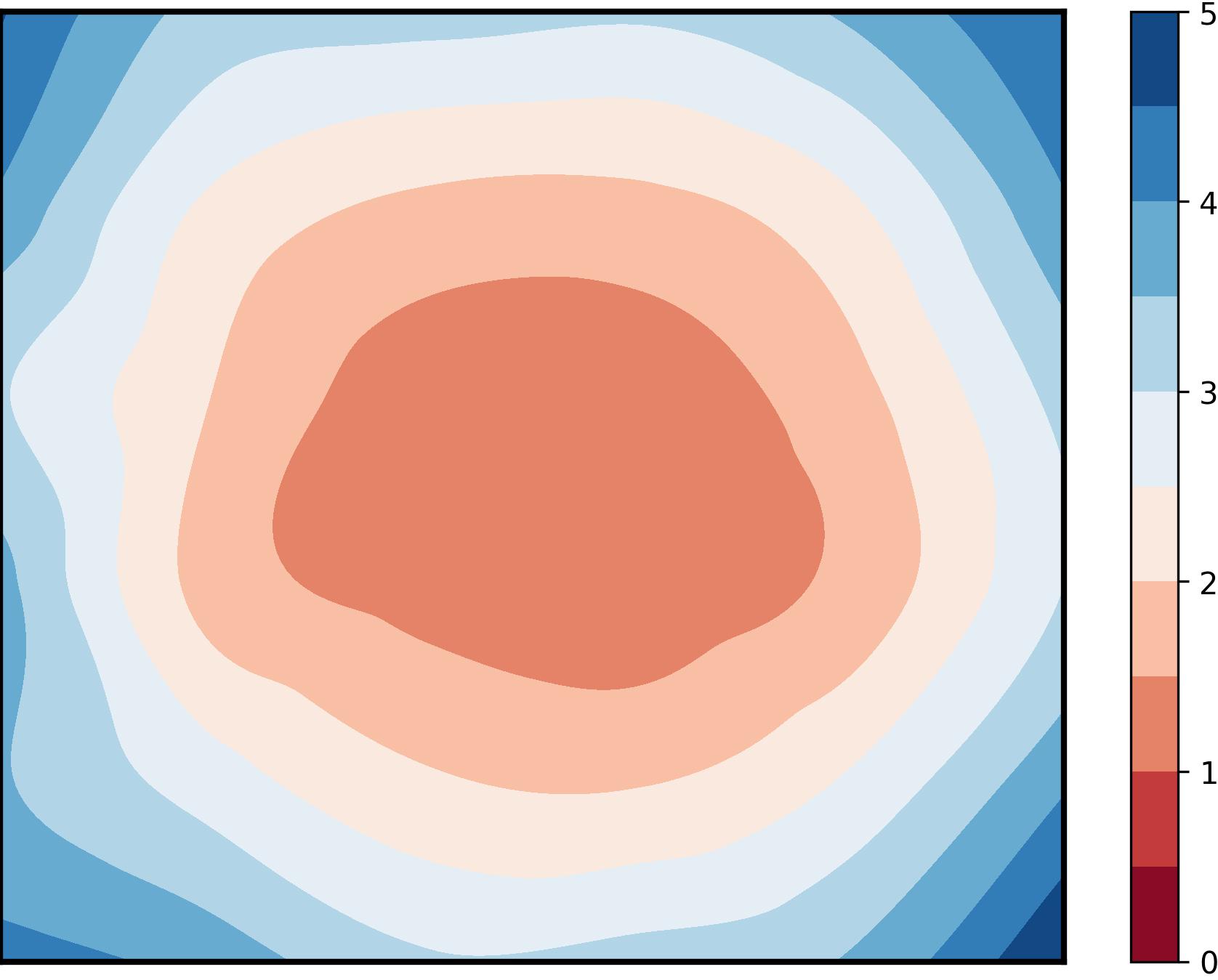}}
	\end{minipage}
	\hfill
	\begin{minipage}[b]{0.16\textwidth}
		\centering
		\subcaptionbox{FedLESAM}{\includegraphics[width=\textwidth]{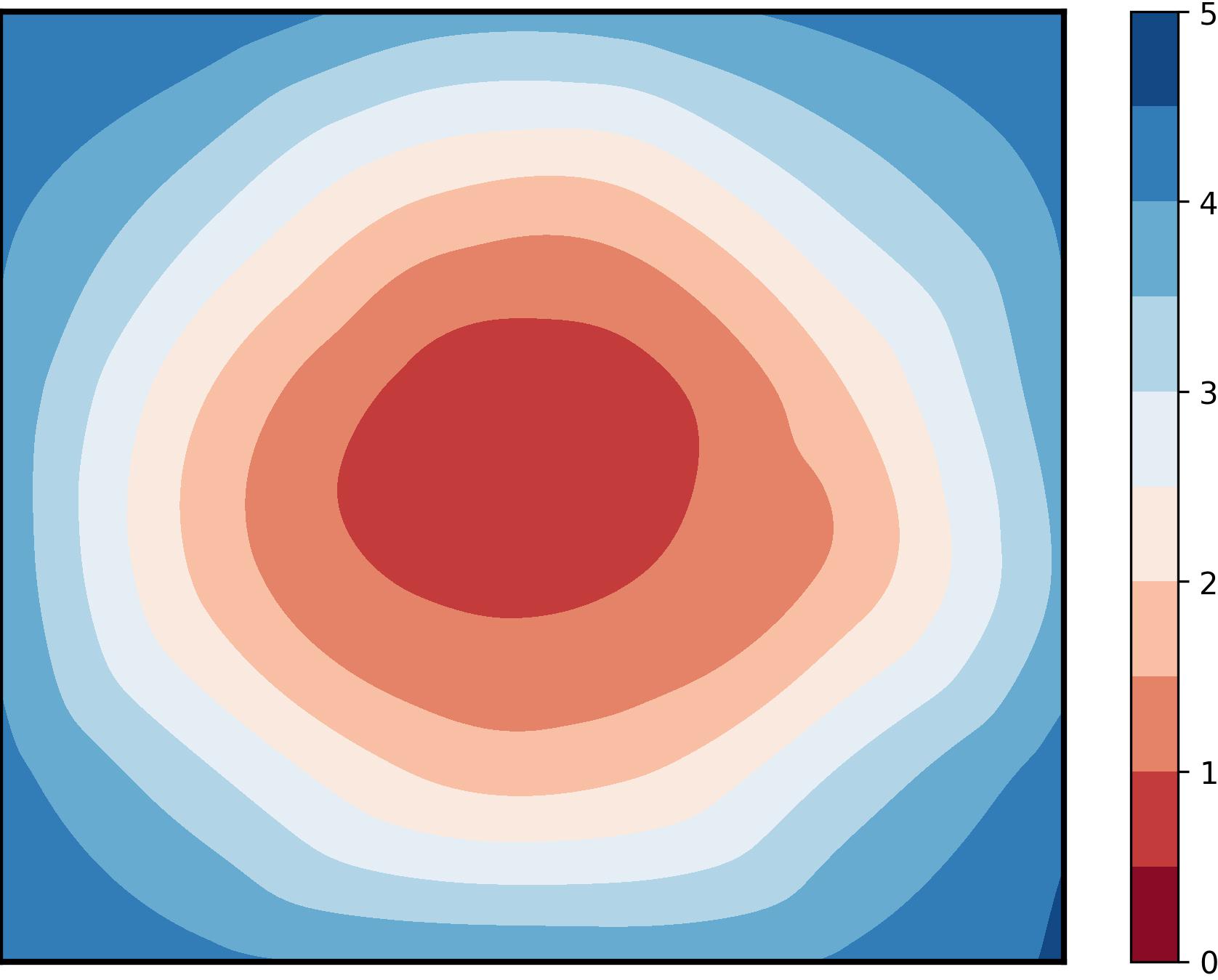}}
	\end{minipage}
	\hfill
	\begin{minipage}[b]{0.16\textwidth}
		\centering
        		\subcaptionbox{FedAvg}{\includegraphics[width=\textwidth]{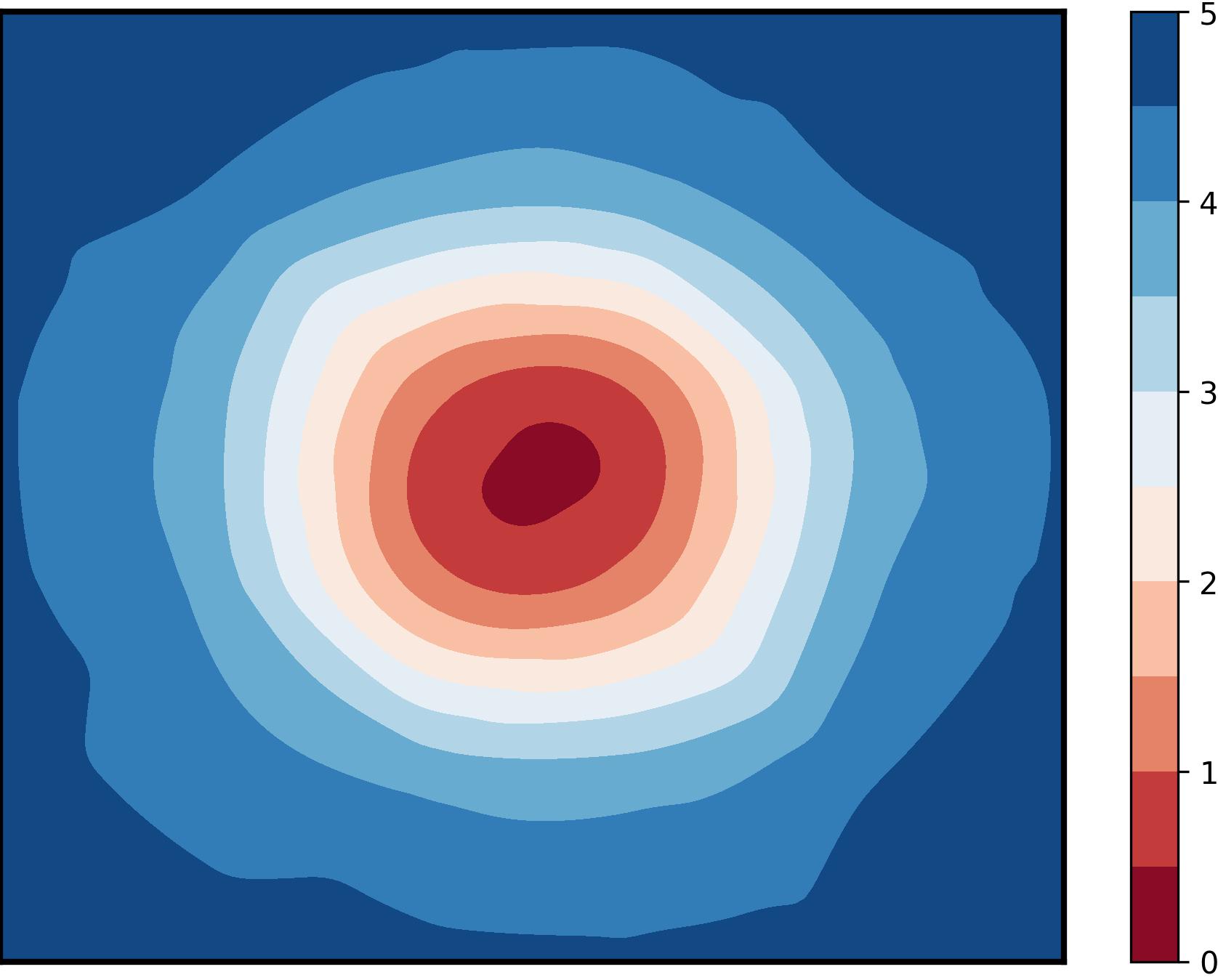}}
        \end{minipage}
	\hfill
	\begin{minipage}[b]{0.16\textwidth}
		\centering
        \subcaptionbox{\textbf{FedNSAM} (ours)}{\includegraphics[width=\textwidth]{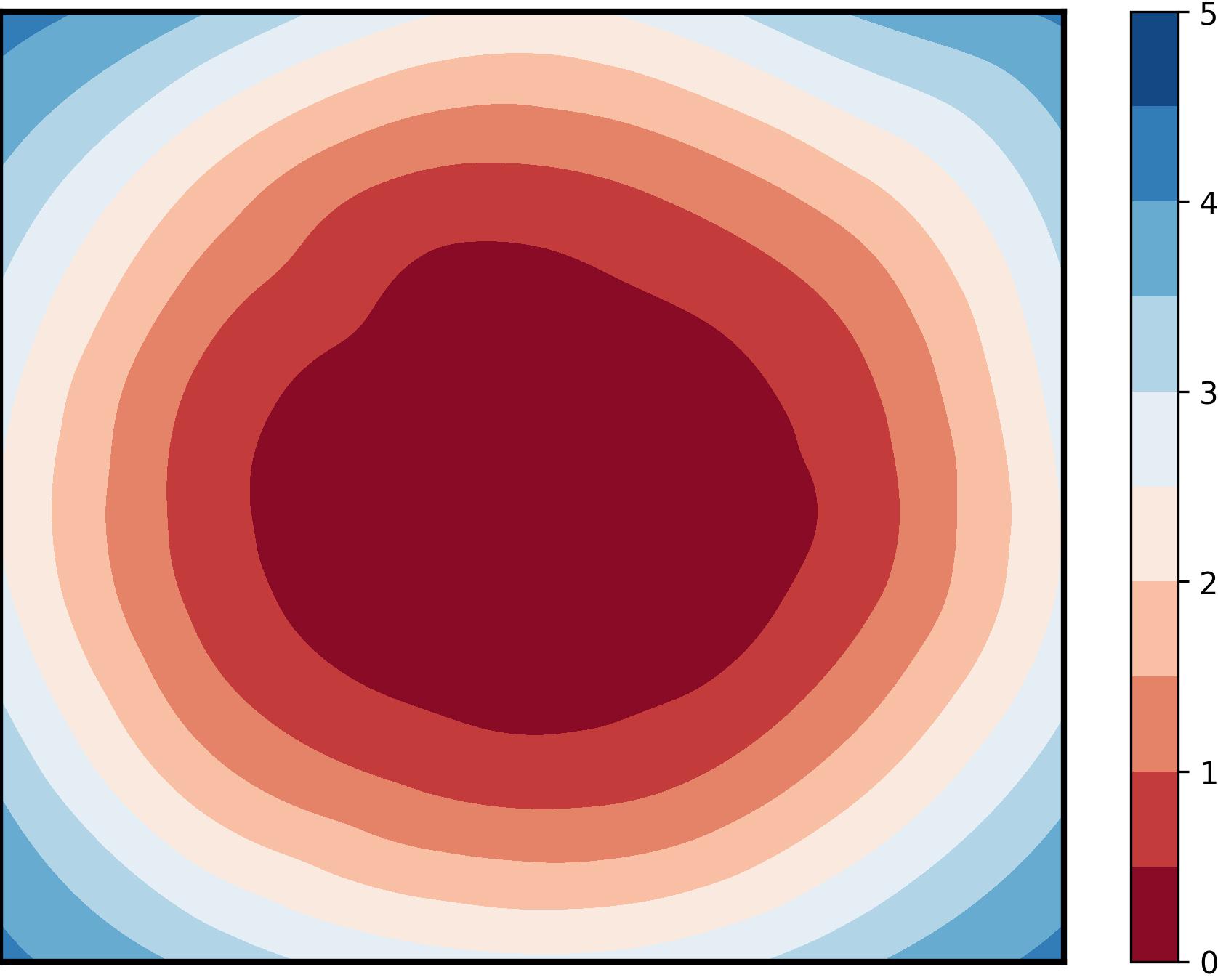}}
	\end{minipage}
\hfill
    \setlength{\abovecaptionskip}{2pt} 
	\caption{Global training loss surfaces plots (a,b,c,d,e,f) and global testing loss surfaces plots (g,h,i,j,k,l) for \textbf{FedNSAM} and other baselines in different settings on Dirichlet-0.1 of CIFAR100 datasets with ResNet-18. 
		\textbf{FedNSAM} could approach a more general and flat loss landscape which efficiently improves the generalization performance in FL.}
	\label{figure 7}
\end{figure*}

\subsection{Participation Rate and Heterogeneity Levels}
To validate the stability of our algorithm, we conducted multiple sets of experiments on the ResNet-18 model and the CIFAR100 dataset under different client participation rates and different data heterogeneity levels.



\textbf{Impact of heterogeneity}: Via sampling with replacement, not only is there a data imbalance between local clients, but also the number of samples among categories in the global dataset is also very different, as shown in Figure  \ref{fig:convergence_all}. With the high heterogeneity (Dirichlet-0.1), several baselines are greatly affected. We select the variance Dirichlet coefficient 0.1, 0.3, and 0.6 on the $10 \%$ participation. In detail, on CIFAR-100, when the Dirichle coefficient is from 0.6 to 0.1, \textbf{FedNSAM} drops from $66.04\%$ to $58.53 \%$, while FedSAM drops from $ 47.83\%$ to $40.18 \%$, which drops $ 7.65 \%$.   \textbf{FedNSAM} achieves good performance on the strong heterogeneous dataset with high stability on different setups.

\textbf{Impact of partial participation}:  To fairly compare with the baselines, we froze the other hyperparameters selections as shown in Figure  \ref{fig:convergence_all} and Table  \ref{tab:merged}. When the participation ratio decreases from $10 \%$ to $5 \%$ and $2 \%$ on CIFAR-100, \textbf{FedNSAM} still achieves excellent performance, which drops from $66.04\%$ to $56.92\%$ on Dirichlet-0.6, \textbf{FedNSAM} maintains a high level of generalization on $2\%$ participation, which achieves the accuracy $56.92\%$.

\subsection{Comparison with FedSAM and Its Variants}
To better evaluate the effectiveness of the proposed global Nesterov extrapolation, we compare \textbf{FedNSAM} with several representative FedSAM-based methods, including FedSAM, MoFedSAM, FedGAMMA, and FedLESAM. We visualize their global training and testing loss landscapes on the CIFAR100 dataset (Dirichlet-0.1) using ResNet-18, as shown in Figure~\ref{figure 7}.
In the case of high data heterogeneity, FedSAM and its variants generally fail to obtain sufficiently flat global minima. Specifically, FedSAM exhibits sharp loss surfaces due to inconsistent local perturbations, while MoFedSAM and FedGAMMA incorporate update correction but still struggle to improve global flatness. FedLESAM alleviates part of this inconsistency by estimating global perturbations, but its estimation remains inaccurate, limiting its effectiveness.
In contrast, \textbf{FedNSAM} achieves significantly flatter global loss landscapes in both training and testing, as shown in Figure~\ref{figure 7} (f, l). This highlights its ability to harmonize local training dynamics and reduce sharpness more effectively under heterogeneous settings.



\begin{table}[tb]
	\centering
	\setlength{\tabcolsep}{4pt}
	\caption{Comparison different data heterogeneity of testing accuracy (\%) of ablation on variants on  CIFAR100.}
	\begin{tabular}{llccccccc}
		\midrule[1.5pt]
		&\multicolumn{4}{c}{CIFAR100}\\ 
		\cmidrule(lr){2-5} 
		Method & \multicolumn{2}{c}{Dirichlet-0.1}  & \multicolumn{2}{c}{Dirichlet-0.6} \\
		\cmidrule(lr){2-3} \cmidrule(lr){4-5}
		& Acc & Rounds & Acc& Rounds\\
		& 1000R & 55\%  &1000R & 55\%\\
        \midrule
		FedAvg  &45.81$\pm0.35$ &1000+    &54.25$\pm0.19$ &1000+  \\
        \rowcolor{LightRed}
		\textbf{FedNSAM}  &\textbf{58.53}$\pm0.27$&  695    &\textbf{66.04}$\pm0.22$&  316  \\
		\midrule[0.25pt]
		SCAFFOLD   &49.02$\pm0.29$ &1000+   &54.12$\pm0.40$ &1000+\\
        \rowcolor{LightRed}
		\textbf{FedNSAM}-S  &\textbf{59.63}$\pm0.32$&  676 &\textbf{66.24}$\pm0.18$&  301  \\
		\midrule[0.25pt]
		FedDyn       &53.53$\pm0.16$&1000+    &56.95$\pm0.30$ &865  \\
        \rowcolor{LightRed}
		\textbf{FedNSAM}-D  &\textbf{60.17}$\pm0.34$&  663 &\textbf{66.48}$\pm0.25$&  296   \\
		\midrule[1.5pt]
	\end{tabular}
	\label{table 6}
\end{table}	
\subsection{Ablation on \textbf{FedNSAM} Variants} 
We design \textbf{FedNSAM} under FedAvg and two enhanced methods under Scaffold and FedDyn, named \textbf{FedNSAM}-S and \textbf{FedNSAM}-D, respectively. Therefore, we show the average performance of all variants on CIFAR100 in Table \ref{table 6}. All the variants can achieve an extensive improvement compared to their base methods, especially a notable $9.53\%$ improvement on CIFAR100 of \textbf{FedNSAM}-D. 

\subsection{Effect of Nesterov Momentum}
As shown in Table \ref{table 7}, we can find that the proposed Nesterov momentum acceleration consistently improves the accuracy and reduces the number of rounds required for convergence across all the methods. For instance, FedAvg with Nesterov momentum achieved an accuracy of 54.23\% on Dirichlet-0.1 and 60.26\% on Dirichlet-0.6, outperforming FedAvg without Nesterov momentum. Similarly, FedSAM+Nesterov achieved higher accuracy than FedSAM, with a significant reduction in the required rounds. However, even without Nesterov momentum, \textbf{FedNSAM} performed very well, achieving an accuracy of 55.86\% on Dirichlet-0.1 and 60.11\% on Dirichlet-0.6, with fewer rounds than other methods. This suggests that our algorithm is inherently effective, even without the acceleration of Nesterov momentum. When Nesterov momentum was applied, the performance of \textbf{FedNSAM} can be further improved, achieving the highest accuracy of 58.53\% and 66.04\% on Dirichlet-0.1 and Dirichlet-0.6, respectively, with a significant reduction in rounds.

\begin{table}[tb]
	\centering
	\setlength{\tabcolsep}{0pt}
	\caption{Comparison different data heterogeneity of  testing accuracy (\%) of each algorithm with Nesterov on  CIFAR100.}
	\begin{tabular}{p{3cm}lccccccc}
		\midrule[1.5pt]
		&\multicolumn{4}{c}{CIFAR100}\\ 
		\cmidrule(lr){2-5} 
		Method & \multicolumn{2}{c}{Dirichlet-0.1}  & \multicolumn{2}{c}{Dirichlet-0.6} \\
		\cmidrule(lr){2-3} \cmidrule(lr){4-5}
		& Acc & Rounds & Acc& Rounds\\
		& 1000R & 55\%  &1000R & 55\%\\
        \midrule
		FedAvg  &45.81$\pm0.32$ &1000+    &54.25$\pm0.15$ &1000+  \\
        \rowcolor{LightRed}
		FedAvg+Nesterov  &\textbf{54.23}$\pm0.28$&  1000+    &\textbf{60.26}$\pm0.39$&  512  \\
		FedSAM   &40.18$\pm0.27$ &1000+   &47.83$\pm0.09$ &1000+\\
        \rowcolor{LightRed}
		FedSAM+Nesterov  &\textbf{55.88}$\pm0.34$&  876 &\textbf{61.02}$\pm0.09$&  482  \\
		\textbf{FedNSAM} - Nesterov       &55.86$\pm0.40$&895   &60.11$\pm0.19$ &488  \\
        \rowcolor{LightRed}
		\textbf{FedNSAM} + Nesterov  &\textbf{58.53}$\pm0.26$&  659 &\textbf{66.04}$\pm0.11$&  316   \\
		\midrule[1.5pt]
	\end{tabular}
	\label{table 7}
\end{table}

\section{Conclusion}
We revisited Sharpness-Aware Minimization (SAM) in the federated learning setting through the novel lens of \textbf{flatness distance}, which captures the divergence of local optimization landscapes under data heterogeneity. We can observe that increasing heterogeneity exacerbates flatness distance, leading to sharper and less generalizable global minima. To address these challenges, we proposed a novel principled FL algorithm (called \textbf{FedNSAM}) that integrates global Nesterov extrapolation into the SAM framework to align client-level flat regions. Extensive experiments demonstrated that \textbf{FedNSAM} significantly outperforms existing FL baselines, especially under non-IID conditions, achieving superior generalization and convergence.

\section*{Acknowledgments}
This work was supported by the National Natural Science Foundation of China (No.\
62276182), Peng Cheng Lab Program (No. PCL2023A08), Tianjin Natural Science Foundation (Nos.\ 24JCYBJC01230, 24JCYBJC01460), and Tianjin Municipal Education Commission Research Plan (No.\
2024ZX008).

\newpage

\bibliographystyle{ACM-Reference-Format}
\balance
\bibliography{sample-base}

@STRING{health = "ACM Transactions on Computing for Healthcare"}

@STRING{tist = "ACM Transactions on Intelligent Systems and Technology"}

@String{Computing = "Computing" }

@String{Computer = "{IEEE} Computer" }

@String{Springer = "Springer-Verlag" }

@ArtifactSoftware{R,
    title = {R: A Language and Environment for Statistical Computing},
    author = {{R Core Team}},
    organization = {R Foundation for Statistical Computing},
    address = {Vienna, Austria},
    year = {2019},
    url = {https://www.R-project.org/},
}

@inproceedings{karimireddy2020scaffold,
	title={Scaffold: Stochastic controlled averaging for federated learning},
	author={Karimireddy, Sai Praneeth and Kale, Satyen and Mohri, Mehryar and Reddi, Sashank and Stich, Sebastian and Suresh, Ananda Theertha},
	booktitle={International Conference on Machine Learning},
	pages={5132--5143},
	year={2020},
	organization={PMLR}
}

@inproceedings{mcmahan2017communication,
	title={Communication-efficient learning of deep networks from decentralized data},
	author={McMahan, Brendan and Moore, Eider and Ramage, Daniel and Hampson, Seth and y Arcas, Blaise Aguera},
	booktitle={Artificial intelligence and statistics},
	pages={1273--1282},
	year={2017},
	organization={PMLR}
}

@article{hsu2019measuring,
	title={Measuring the effects of non-identical data distribution for federated visual classification},
	author={Hsu, Tzu-Ming Harry and Qi, Hang and Brown, Matthew},
	journal={arXiv preprint arXiv:1909.06335},
	year={2019}
}

@article{krizhevsky2009learning,
	title={Learning multiple layers of features from tiny images},
	author={Krizhevsky, Alex and Hinton, Geoffrey and others},
	year={2009},
	publisher={Toronto, ON, Canada}
}

@article{lecun2015lenet,
	title={LeNet-5, convolutional neural networks},
	author={LeCun, Yann and others},
	journal={URL: http://yann. lecun. com/exdb/lenet},
	volume={20},
	number={5},
	pages={14},
	year={2015}
}

@article{simonyan2014very,
	title={Very deep convolutional networks for large-scale image recognition},
	author={Simonyan, Karen and Zisserman, Andrew},
	journal={arXiv preprint arXiv:1409.1556},
	year={2014}
}

@inproceedings{he2016deep,
	title={Deep residual learning for image recognition},
	author={He, Kaiming and Zhang, Xiangyu and Ren, Shaoqing and Sun, Jian},
	booktitle={Proceedings of the IEEE conference on computer vision and pattern recognition},
	pages={770--778},
	year={2016}
}

@article{rieke2020future,
	title={The future of digital health with federated learning},
	author={Rieke, Nicola and Hancox, Jonny and Li, Wenqi and Milletari, Fausto and Roth, Holger R and Albarqouni, Shadi and Bakas, Spyridon and Galtier, Mathieu N and Landman, Bennett A and Maier-Hein, Klaus and others},
	journal={NPJ digital medicine},
	volume={3},
	number={1},
	pages={119},
	year={2020},
	publisher={Nature Publishing Group UK London}
}

@article{antunes2022federated,
	title={Federated learning for healthcare: Systematic review and architecture proposal},
	author={Antunes, Rodolfo Stoffel and Andr{\'e} da Costa, Cristiano and K{\"u}derle, Arne and Yari, Imrana Abdullahi and Eskofier, Bj{\"o}rn},
	journal={ACM Transactions on Intelligent Systems and Technology (TIST)},
	volume={13},
	number={4},
	pages={1--23},
	year={2022},
	publisher={ACM New York, NY}
}

@inproceedings{byrd2020differentially,
	title={Differentially private secure multi-party computation for federated learning in financial applications},
	author={Byrd, David and Polychroniadou, Antigoni},
	booktitle={Proceedings of the First ACM International Conference on AI in Finance},
	pages={1--9},
	year={2020}
}

@article{dosovitskiy2020image,
	title={An image is worth 16x16 words: Transformers for image recognition at scale},
	author={Dosovitskiy, Alexey and Beyer, Lucas and Kolesnikov, Alexander and Weissenborn, Dirk and Zhai, Xiaohua and Unterthiner, Thomas and Dehghani, Mostafa and Minderer, Matthias and Heigold, Georg and Gelly, Sylvain and others},
	journal={arXiv preprint arXiv:2010.11929},
	year={2020}
}

@article{le2015tiny,
	title={Tiny imagenet visual recognition challenge},
	author={Le, Ya and Yang, Xuan},
	journal={CS 231N},
	volume={7},
	number={7},
	pages={3},
	year={2015}
}

@inproceedings{qu2022generalized,
	title={Generalized federated learning via sharpness aware minimization},
	author={Qu, Zhe and Li, Xingyu and Duan, Rui and Liu, Yao and Tang, Bo and Lu, Zhuo},
	booktitle={International Conference on Machine Learning},
	pages={18250--18280},
	year={2022},
	organization={PMLR}
}

@article{foret2020sharpness,
	title={Sharpness-aware minimization for efficiently improving generalization},
	author={Foret, Pierre and Kleiner, Ariel and Mobahi, Hossein and Neyshabur, Behnam},
	journal={arXiv preprint arXiv:2010.01412},
	year={2020}
}

@article{dai2023fedgamma,
	title={Fedgamma: Federated learning with global sharpness-aware minimization},
	author={Dai, Rong and Yang, Xun and Sun, Yan and Shen, Li and Tian, Xinmei and Wang, Meng and Zhang, Yongdong},
	journal={IEEE Transactions on Neural Networks and Learning Systems},
	year={2023},
	publisher={IEEE}
}

@inproceedings{kim2024communication,
	title={Communication-efficient federated learning with accelerated client gradient},
	author={Kim, Geeho and Kim, Jinkyu and Han, Bohyung},
	booktitle={Proceedings of the IEEE/CVF Conference on Computer Vision and Pattern Recognition},
	pages={12385--12394},
	year={2024}
}

@article{fan2024locally,
	title={Locally Estimated Global Perturbations are Better than Local Perturbations for Federated Sharpness-aware Minimization},
	author={Fan, Ziqing and Hu, Shengchao and Yao, Jiangchao and Niu, Gang and Zhang, Ya and Sugiyama, Masashi and Wang, Yanfeng},
	journal={arXiv preprint arXiv:2405.18890},
	year={2024}
}

@inproceedings{sun2023dynamic,
	title={Dynamic regularized sharpness aware minimization in federated learning: Approaching global consistency and smooth landscape},
	author={Sun, Yan and Shen, Li and Chen, Shixiang and Ding, Liang and Tao, Dacheng},
	booktitle={International Conference on Machine Learning},
	pages={32991--33013},
	year={2023},
	organization={PMLR}
}

@inproceedings{liu2021swin,
	title={Swin transformer: Hierarchical vision transformer using shifted windows},
	author={Liu, Ze and Lin, Yutong and Cao, Yue and Hu, Han and Wei, Yixuan and Zhang, Zheng and Lin, Stephen and Guo, Baining},
	booktitle={Proceedings of the IEEE/CVF international conference on computer vision},
	pages={10012--10022},
	year={2021}
}

@inproceedings{li2018federated,
	title={Federated optimization in heterogeneous networks},
	author={Li, Tian and Sahu, Anit Kumar and Zaheer, Manzil and Sanjabi, Maziar and Talwalkar, Ameet and Smith, Virginia},
	booktitle={Proceedings of Machine Learning and Systems},
	year={2020}
}

@inproceedings{caldarola2022improving,
	title={Improving generalization in federated learning by seeking flat minima},
	author={Caldarola, Debora and Caputo, Barbara and Ciccone, Marco},
	booktitle={European Conference on Computer Vision},
	pages={654--672},
	year={2022},
	organization={Springer}
}

@book{nesterov2013introductory,
	title={Introductory lectures on convex optimization: A basic course},
	author={Nesterov, Yurii},
	volume={87},
	year={2013},
	publisher={Springer Science \& Business Media}
}

@inproceedings{li2024friendly,
	title={Friendly sharpness-aware minimization},
	author={Li, Tao and Zhou, Pan and He, Zhengbao and Cheng, Xinwen and Huang, Xiaolin},
	booktitle={Proceedings of the IEEE/CVF Conference on Computer Vision and Pattern Recognition},
	pages={5631--5640},
	year={2024}
}

@inproceedings{liu2024fedbcgd,
  title={Fedbcgd: Communication-efficient accelerated block coordinate gradient descent for federated learning},
  author={Liu, Junkang and Shang, Fanhua and Liu, Yuanyuan and Liu, Hongying and Li, Yuangang and Gong, YunXiang},
  booktitle={Proceedings of the 32nd ACM International Conference on Multimedia},
  pages={2955--2963},
  year={2024}
}

@article{zeng2025FSDrive,
      title={FutureSightDrive: Thinking Visually with Spatio-Temporal CoT for Autonomous Driving},
      author={Shuang Zeng and Xinyuan Chang and Mengwei Xie and Xinran Liu and Yifan Bai and Zheng Pan and Mu Xu and Xing Wei},
      journal={arXiv preprint arXiv:2505.17685},
      year={2025}
      }

@inproceedings{liuimproving,
  title={Improving Generalization in Federated Learning with Highly Heterogeneous Data via Momentum-Based Stochastic Controlled Weight Averaging},
  author={Liu, Junkang and Liu, Yuanyuan and Shang, Fanhua and Liu, Hongying and Liu, Jin and Feng, Wei},
  booktitle={Forty-second International Conference on Machine Learning},
  year={2025}
}

@misc{dai2025securetugofwarsectowiterative,
      title={Secure Tug-of-War (SecTOW): Iterative Defense-Attack Training with Reinforcement Learning for Multimodal Model Security}, 
      author={Muzhi Dai and Shixuan Liu and Zhiyuan Zhao and Junyu Gao and Hao Sun and Xuelong Li},
      year={2025},
      eprint={2507.22037},
      archivePrefix={arXiv},
      primaryClass={cs.CR},
      url={https://arxiv.org/abs/2507.22037}, 
}

@misc{dai2025captionsrewardscarevlleveraging,
      title={From Captions to Rewards (CAREVL): Leveraging Large Language Model Experts for Enhanced Reward Modeling in Large Vision-Language Models}, 
      author={Muzhi Dai and Jiashuo Sun and Zhiyuan Zhao and Shixuan Liu and Rui Li and Junyu Gao and Xuelong Li},
      year={2025},
      eprint={2503.06260},
      archivePrefix={arXiv},
      primaryClass={cs.CV},
      url={https://arxiv.org/abs/2503.06260}, 
}

@article{yang2025distillation,
  title={Distillation guided deep unfolding network with frequency hierarchical regularization for low-dose CT image denoising},
  author={Yang, Linlin and Liu, Hongying and Liu, Yuanyuan and Shang, Fanhua and Wan, Liang and Jiao, Licheng},
  journal={Neurocomputing},
  pages={130535},
  year={2025},
  publisher={Elsevier}
}

@article{liu2025fedadamw,
  title={FedAdamW: A Communication-Efficient Optimizer with Convergence and Generalization Guarantees for Federated Large Models},
  author={Liu, Junkang and Shang, Fanhua and Zhu, Kewen and Liu, Hongying and Liu, Yuanyuan and Liu, Jin},
  journal={arXiv preprint arXiv:2510.27486},
  year={2025}
}

@inproceedings{liu2025consistency,
  title={Consistency of local and global flatness for federated learning},
  author={Liu, Junkang and Shang, Fanhua and Tian, Yuxuan and Liu, Hongying and Liu, Yuanyuan},
  booktitle={Proceedings of the 33rd ACM International Conference on Multimedia},
  pages={3875--3883},
  year={2025}
}

@article{liu2025fedmuon,
  title={FedMuon: Accelerating Federated Learning with Matrix Orthogonalization},
  author={Liu, Junkang and Shang, Fanhua and Zhou, Junchao and Liu, Hongying and Liu, Yuanyuan and Liu, Jin},
  journal={arXiv preprint arXiv:2510.27403},
  year={2025}
}

@article{liu2025dp,
  title={DP-FedPGN: Finding Global Flat Minima for Differentially Private Federated Learning via Penalizing Gradient Norm},
  author={Liu, Junkang and Tian, Yuxuan and Shang, Fanhua and Liu, Yuanyuan and Liu, Hongying and Zhou, Junchao and Ding, Daorui},
  journal={arXiv preprint arXiv:2510.27504},
  year={2025}
}

@misc{liu2026tamingpreconditionerdriftunlocking,
      title={Taming Preconditioner Drift: Unlocking the Potential of Second-Order Optimizers for Federated Learning on Non-IID Data}, 
      author={Junkang Liu and Fanhua Shang and Hongying Liu and Jin Liu and Weixin An and Yuanyuan Liu},
      year={2026},
      eprint={2602.19271},
      archivePrefix={arXiv},
      primaryClass={cs.LG},
      url={https://arxiv.org/abs/2602.19271}, 
}

\newpage

\onecolumn

\begin{algorithm}[tb]
	\caption{\textbf{FedNSAM} Algorithm}
	\begin{algorithmic}[1]
		\STATE Input: $\beta, \lambda$, initial server model $\boldsymbol{\theta}^0$, number of clients $N$, number of communication rounds $T$, number of local iterations $K$, local learning rate $\eta$
		\STATE Initialize global momentum $\boldsymbol{m}^0=0$, global model $\boldsymbol{\theta}^0$
		\FOR{each round $t=1$ to $T$}
		\FOR{each selected client $i=1$ to $S$}
		\FOR{$k=0, \ldots, K$ local update}
		\STATE Client $i$  Update local model $\boldsymbol{\theta}_i^t$,
		\STATE $\boldsymbol{\theta}_{i, k+1/4}^t=\boldsymbol{\theta}_{i, k}^t+\lambda \boldsymbol{m}^t$ $  \triangleright$ Nesterov extrapolation
		\STATE $\boldsymbol{\delta}_{i, k}^t=\rho \frac{-\boldsymbol{m}^t}{\|\boldsymbol{m}^t\|}$ $  \triangleright$ Perturbation calculation
		\STATE $\boldsymbol{\theta}_{i, k+1/2}^t=\boldsymbol{\theta}_{i, k+1/4}^t+\boldsymbol{\delta}_{i, k}^t$ $  \triangleright$ Perturbation model
		\STATE $\boldsymbol{\theta}_{i, k+1}^t\! =\! \boldsymbol{\theta}_{i, k}^t\!-\!\eta \nabla F_i\left(\boldsymbol{\theta}_{i, k+1/2}^t; \zeta_i\right)$ $\triangleright$ Updating
		\ENDFOR
		\ENDFOR
		\STATE$\boldsymbol{\Delta}_i^t=\boldsymbol{\theta}_{i, K}^t-\boldsymbol{\theta}_{i, 0}^t$
		\STATE Client sends $\boldsymbol{\Delta}_i^t$ back to the server
		\STATE Server averages models: 
		\STATE$\boldsymbol{\Delta}^t=\frac{1}{S}\sum_{i \in S_t} \boldsymbol{\Delta}_i^t $\\
		\STATE$\boldsymbol{m}^t=\lambda \boldsymbol{m}^{t-1}+\boldsymbol{\Delta}^t$ $  \triangleright$ Nesterov momentum\\
		\STATE$\boldsymbol{\theta}^t=\boldsymbol{\theta}^{t-1}+\boldsymbol{m}^t$
		\ENDFOR
	\end{algorithmic}
\end{algorithm}

\section{Appendix A: Basic Assumptions}

\begin{Assumption}\label{CA1}
	We first state a few assumptions for the local loss functions $F_i(\cdot)$ . First, the local function $F_i(\cdot)$ is assumed to be $L$-smooth for all $C_i \in\left\{C_1, \ldots, C_N\right\}$, i.e.,
	
	$$
	\left\|\nabla F_i(x)-\nabla F_i(y)\right\| \leq L\|x-y\| \forall x, y
	$$

	This also implies
	
	$$
	F_i(y) \leq F_i(x)+\left\langle\nabla F_i(x), y-x\right\rangle+\frac{L}{2}\|y-x\|^2 .
	$$
	
\end{Assumption}

\begin{Assumption}\label{CA2}
	Second, we assume the stochastic gradient of the local loss function $\nabla f_i(x):=\nabla F_i\left(x ; \mathcal{D}_i\right)$ is unbiased and possesses a bounded variance, i.e. $\mathbb{E}_{\mathcal{D}_i}\left[\left\|\nabla f_i(x)-\nabla F_i(x)\right\|^2\right] \leq \sigma^2$. 
\end{Assumption}

\begin{Assumption}[Bounded heterogeneity]\label{CA3}
	Third, we assume the average norm of local gradients is bounded
	by a function of the global gradient magnitude as $\frac{1}{N} \sum_{i=1}^N\left\|\nabla F_i(x)\right\|^2 \leq \sigma_g^2+B^2\|\nabla F(x)\|^2$, where $\sigma_g \geq 0$ and $B \geq 1$. Based on the above assumptions, we derive the following asymptotic convergence bound of FedACG.
	The dissimilarity of $F_i(\boldsymbol{\boldsymbol{\theta}})$ and $f(\boldsymbol{\boldsymbol{\theta}})$ is bounded as follows:
	\begin{equation}
		\frac{1}{m} \sum_{i=1}^m\left\|\nabla F_i(\boldsymbol{\boldsymbol{\theta}})-\nabla F(\boldsymbol{\boldsymbol{\theta}})\right\|^2 \leq \sigma_g^2.
	\end{equation}
\end{Assumption}
$$
\begin{aligned}
	& \frac{1}{N} \sum_{i=1}^N\left\|\nabla F_i(x)-\nabla F(x)+\nabla F(x)\right\|^2 \\
	& \leq \frac{1}{N} \sum_{i=1}^N\left\|\nabla F_i(x)-\nabla F(x)\right\|^2+2 \frac{1}{N} \sum_{i=1}^N\left\langle\nabla F_i(x)-\nabla F(x), \nabla F(x)\right\rangle+\frac{1}{N} \sum_{i=1}^N\|\nabla F(x)\|^2 \\
	& \leq \frac{1}{N} \sum_{i=1}^N\left\|\nabla F_i(x)-\nabla F(x)\right\|^2+\frac{1}{N} \sum_{i=1}^N\|\nabla F(x)\|^2 \\
	& \leq \sigma_g^2+\|\nabla F(x)\|^2
\end{aligned}
$$

\section{Appendix B: Main Lemmas}

We present several technical lemmas that are useful for subsequent proofs.
\begin{lemma} \label{CL1}
	(relaxed triangle inequality). Let $\left\{v_1, \ldots, v_\tau\right\}$ be $\tau$ vectors in $\mathbb{R}^d$. Then the following are true: (1) $\left\|v_i+v_j\right\|^2 \leq$ $(1+a)\left\|v_i\right\|^2+\left(1+\frac{1}{a}\right)\left\|v_j\right\|^2$ for any $a>0$, and (2) $\left\|\sum_{i=1}^\tau v_i\right\|^2 \leq \tau \sum_{i=1}^\tau\left\|v_i\right\|^2$.
	
\end{lemma}

\begin{lemma}\label{CL2}
	(sub-linear convergence rate). For every non-negative sequence $\left\{d_{r-1}\right\}_{r \geq 1}$ and any parameters $\eta_{\max } \geq 0, c \geq 0$, $R \geq 0$, there exists a constant step-size $\eta \leq \eta_{\max }$ and weights $w_r=1$ such that,
	$$
	\Psi_R:=\frac{1}{R+1} \sum_{r=1}^{R+1}\left(\frac{d_{r-1}}{\eta}-\frac{d_r}{\eta}+c_1 \eta+c_2 \eta^2\right) \leq \frac{d_0}{\eta_{\max }(R+1)}+\frac{2 \sqrt{c_1 d_0}}{\sqrt{R+1}}+2\left(\frac{d_0}{R+1}\right)^{\frac{2}{3}} c_2^{\frac{1}{3}}
	$$
\end{lemma}

\begin{proof}
Unrolling the sum, we can simplify
$$
\Psi_R \leq \frac{d_0}{\eta(R+1)}+c_1 \eta+c_2 \eta^2
$$
The lemma can be established through the adjustment of $\eta$. We consider the following two cases based on the magnitudes of R and $\eta_{\max }$ :
- When $R+1 \leq \frac{d_0}{c_1 \eta_{\max }^2}$ and $R+1 \leq \frac{d_0}{c_2 \eta_{\max }^3}$, selecting $\eta=\eta_{\max }$ satisfies
$$
\Psi_R \leq \frac{d_0}{\eta_{\max }(R+1)}+c_1 \eta_{\max }+c_2 \eta_{\max }^2 \leq \frac{d_0}{\eta_{\max }(R+1)}+\frac{\sqrt{c_1 d_0}}{\sqrt{R+1}}+\left(\frac{d_0}{R+1}\right)^{\frac{2}{3}} c_2^{\frac{1}{3}}
$$
In the other case, we have $\eta_{\max }^2 \geq \frac{d_0}{c_1(R+1)}$ or $\eta_{\max }^3 \geq \frac{d_0}{c_2(R+1)}$. Choosing $\eta=\min \left\{\sqrt{\frac{d_0}{c_1(R+1)}}, \sqrt[3]{\frac{d_0}{c_2(R+1)}}\right\}$ satisfies
$$
\Psi_R \leq \frac{d_0}{\eta(R+1)}+c \eta=\frac{2 \sqrt{c_1 d_0}}{\sqrt{R+1}}+2 \sqrt[3]{\frac{d_0^2 c_2}{(R+1)^2}}
$$
\end{proof}
\begin{lemma}\label{CL3}
	(separating mean and variance). Given a set of $\tau$ random variables $\left\{\mathbf{x}_1, \ldots, \mathbf{x}_\tau\right\}$ in $\mathbb{R}^d$, where $\mathbb{E}\left[\mathbf{x}_i \mid \mathbf{x}_{i-1}, \ldots \mathbf{x}_1\right]=$ $\xi_i$ and $\mathbb{E}\left[\left\|\mathbf{x}_i-\xi_i\right\|^2\right] \leq \sigma^2$ represent their conditional mean and variance, respectively, the variables $\left\{\mathbf{x}_i-\xi_i\right\}$ form a martingale difference sequence. Based on this setup, the following holds
	
	$$
	\mathbb{E}\left[\left\|\sum_{i=1}^\tau \mathbf{x}_i\right\|^2\right] \leq 2\left\|\sum_{i=1}^\tau \xi_i\right\|^2+2 \tau \sigma^2
	$$
\end{lemma}

\begin{proof}

$$
\begin{aligned}
	\mathbb{E}\left[\left\|\sum_{i=1}^\tau \mathbf{x}_i\right\|^2\right] & \leq 2\left\|\sum_{i=1}^\tau \xi_i\right\|^2+2 \mathbb{E}\left[\left\|\sum_{i=1}^\tau \mathbf{x}_i-\xi_i\right\|^2\right] \\
	& =2\left\|\sum_{i=1}^\tau \xi_i\right\|^2+2 \sum_i \mathbb{E}\left[\left\|\mathbf{x}_i-\xi_i\right\|^2\right] \\
	& \leq 2\left\|\sum_{i=1}^\tau \xi_i\right\|^2+2 \tau \sigma^2
\end{aligned}
$$

The first inequality comes from the relaxed triangle inequality and the following equality holds because $\left\{\mathbf{x}_i-\xi_i\right\}$ forms a martingale difference sequence.

\end{proof}

\begin{lemma}\label{L_perturbation_difference}(Bounded perturbation difference). Let Assumption 1 and 2 hold, given local perturbations $\boldsymbol{\delta}_{i, k}^t(k=0,1, \ldots, E-$ 1) at any step and local perturbation $\boldsymbol{\delta}_{i, 0}^t$ at the first step, the variance of perturbation difference in FedSAM can be bounded as:
	
	$$
	\frac{1}{N} \sum_i \mathbb{E}\left[\left\|\boldsymbol{\delta}_{i, k}^t-\boldsymbol{\delta}_{i, 0}^t\right\|^2\right] \leq 2 K^2 L^2 \eta^2 \rho^2
	$$

	However in our \textbf{FedNSAM}, it is zero since the perturbation is consistent during the local training within a round:
	
	$$
	\frac{1}{N} \sum_i \mathbb{E}\left[\left\|\boldsymbol{\delta}_{i, k}^t-\boldsymbol{\delta}_{i, 0}^t\right\|^2\right]=0
	$$
\end{lemma}

\section{Appendix D: Theoretical Results}

\begin{theorem}[Convergence for non-convex functions]
	Suppose that local functions $\left\{F_i\right\}_{i=1}^N$ are non-convex and L-smooth. By setting $\eta \leq \frac{(1-\lambda)^2}{128 K L}$, \textbf{FedNSAM} satisfies
	
	$$
	\begin{aligned}
		& \min _{t=1, \ldots, T} \mathbb{E}\left\|\nabla F\left(\boldsymbol{\theta}^{t-1}+\lambda m^{t-1}\right)\right\|^2 \\
		& \leq \mathcal{O}\left(\frac{M_1 \sqrt{L D}}{\sqrt{T K\left|S_t\right|}}+\frac{\left(L D(1-\lambda)^2\right)^{\frac{2}{3}} M_2^{\frac{1}{3}}}{(T+1)^{\frac{2}{3}}}+\frac{L D}{T}\right)
	\end{aligned}
	$$
where $M_1^2:=\sigma^2+K\left(1-\frac{\left|S_t\right|}{N}\right) \sigma_g^2, M_2:=\frac{\sigma^2}{K}+\sigma_g^2$, and $D:=\frac{F\left(\boldsymbol{\theta}^0\right)-F\left(\boldsymbol{\theta}^*\right)}{1-\lambda}$.
\end{theorem}

\begin{proof}
Let $z^t=\boldsymbol{\theta}^t+\frac{\lambda}{1-\lambda} \boldsymbol{m}^t$ and $\Phi^t=\boldsymbol{\theta}^t+\lambda \boldsymbol{m}^t=\boldsymbol{\theta}^t_{0+1/4}$, 
Note that $z^0=\boldsymbol{\theta}^0$ and $z^t-z^{t-1}=\frac{1}{1-\lambda} \boldsymbol{\delta}^t$, $\boldsymbol{\delta}^t=\frac{\eta }{ S} \sum_{k, C_i} \nabla F_i\left(\boldsymbol{\theta}_{i, k-1+1/2}^t;\zeta_i\right)$. By the smoothness of the function $F(\mathbf{x})$, we have
$$
F\left(z^t\right) \leq F\left(z^{t-1}\right)+\left\langle\nabla F\left(z^{t-1}\right), z^t-z^{t-1}\right\rangle+\frac{L}{2}\left\|z^t-z^{t-1}\right\|^2
$$
By taking the expectation on both sides, we have
\begin{equation}
\begin{aligned}\label{eq C1}
    \mathbb{E}\left[F\left(z^t\right)\right] & \leq \mathbb{E}\left[F\left(z^{t-1}\right)\right]+\frac{1}{1-\lambda} \mathbb{E}\left[\left\langle\nabla F\left(z^{t-1}\right), \boldsymbol{\delta}^t\right\rangle\right]+\frac{L}{2} \mathbb{E}\left[\left\|z^t-z^{t-1}\right\|^2\right] \\
	& =\mathbb{E}\left[F\left(z^{t-1}\right)\right]+\frac{1}{1-\lambda} \mathbb{E}\left[\left\langle\nabla F\left(z^{t-1}\right)-\nabla F\left(\Phi^{t-1}\right), \boldsymbol{\delta}^t\right\rangle\right]+\frac{1}{1-\lambda} \mathbb{E}\left[\left\langle\nabla F\left(\Phi^{t-1}\right), \boldsymbol{\delta}^t\right\rangle\right]+\frac{L}{2(1-\lambda)^2} \mathbb{E}\left[\left\|\boldsymbol{\delta}^t\right\|^2\right]
\end{aligned}
\end{equation}
We note that
\begin{equation}\label{eq C2}
\begin{aligned}
	\frac{1}{1-\lambda} \mathbb{E}\left[\left\langle\nabla F\left(z^{t-1}\right)-\nabla F\left(\Phi^{t-1}\right), \boldsymbol{\delta}^t\right\rangle\right] & \leq \frac{1-\lambda}{2 \lambda^3 L} \mathbb{E}\left[\left\|\nabla F\left(z^{t-1}\right)-\nabla F\left(\Phi^{t-1}\right)\right\|^2\right]+\frac{\lambda^3 L}{2(1-\lambda)^3} \mathbb{E}\left[\left\|\boldsymbol{\delta}^t\right\|^2\right] \\
	& \leq \frac{(1-\lambda) L}{2 \lambda^3} \mathbb{E}\left[\left\|z^{t-1}-\Phi^{t-1}\right\|^2\right]+\frac{\lambda^3 L}{2(1-\lambda)^3} \mathbb{E}\left[\left\|\boldsymbol{\delta}^t\right\|^2\right] \\
	& \leq \frac{L}{2(1-\lambda)} \mathbb{E}\left[\left\|m^{t-1}\right\|^2\right]+\frac{L}{2(1-\lambda)^3} \mathbb{E}\left[\left\|\boldsymbol{\delta}^t\right\|^2\right]
\end{aligned}
\end{equation}
where the first inequality holds because $\langle a, b\rangle \leq \frac{1}{2}\left(\|a\|^2+\|b\|^2\right)$, while the second inequality follows from the $L$-smoothness. The third inequality follows because $z^t-\Phi^t=\frac{\lambda^2}{1-\lambda} \boldsymbol{m}^t$ and $0 \leq \lambda<1$, $\boldsymbol{\Delta}^t=\frac{\eta}{S} \sum_{k, C_i} \nabla F_i\left(\boldsymbol{\theta}_{i, k-1+1/2}^t; \zeta_i\right)$ ,$\left|S_t\right|=S$.
We also note that	
$$
\begin{aligned}
	\frac{1}{1-\lambda} \mathbb{E}\left[\left\langle\nabla F\left(\Phi^{t-1}\right), \boldsymbol{\delta}^t\right\rangle\right] & =\frac{1}{1-\lambda} \mathbb{E}\left[\left\langle\nabla F\left(\Phi^{t-1}\right), \frac{-\eta K}{K N} \sum_{k, C_i} \nabla F_i\left(\boldsymbol{\theta}_{i, k-1+1/2}^t; \zeta_i\right)\right\rangle\right] \\
	& \leq \frac{\eta K}{2(1-\lambda)}\left(\mathbb{E}\left[\left\|\nabla F\left(\Phi^{t-1}\right)-\frac{1}{K N} \sum_{k, C_i} \nabla F_i\left(\boldsymbol{\theta}_{i, k-1+1/2}^t; \zeta_i\right)\right\|^2\right]-\mathbb{E}\left[\left\|\nabla F\left(\Phi^{t-1}\right)\right\|^2\right]\right) \\
	& \leq \frac{\eta K}{2(1-\lambda)}\left(\frac{L^2}{K N} \sum_{k, C_i} \mathbb{E}\left[\left\|\boldsymbol{\theta}_{i, k-1+1/2}^t-\boldsymbol{\theta}_{i, 0+1/4}^t\right\|^2\right]-\mathbb{E}\left[\left\|\nabla F\left(\Phi^{t-1}\right)\right\|^2\right]\right)
\end{aligned}
$$
where the first inequality holds because $\langle a, b\rangle \leq \frac{1}{2}\|a+b\|^2-\frac{1}{2}\|a\|^2$.
Substituting Eq.(\ref{eq C2}) and Eq.(\ref{eq C3})into Eq. (\ref{eq C1}) yields

$$
\begin{aligned}
	\mathbb{E}\left[F\left(z^t\right)\right] \leq & \mathbb{E}\left[F\left(z^{t-1}\right)\right]+\frac{\eta K}{2(1-\lambda)}\left(\frac{L^2}{K N} \sum_{k, C} \mathbb{E}\left[\left\|\boldsymbol{\theta}_{i, k-1+1/2}^t-\boldsymbol{\theta}_{i, 0+1/4}^t\right\|^2\right]-\mathbb{E}\left[\left\|\nabla F\left(\Phi^{t-1}\right)\right\|^2\right]\right) \\
	& +\frac{L}{2(1-\lambda)} \mathbb{E}\left[\left\|m^{t-1}\right\|^2\right]+\left(\frac{L}{2(1-\lambda)^3}+\frac{L}{2(1-\lambda)^2}\right) \mathbb{E}\left[\left\|\boldsymbol{\Delta}^t\right\|^2\right]
\end{aligned}
$$
By rearranging the inequality above, we have

$$
\begin{aligned}
	\frac{\eta K}{2(1-\lambda)} \mathbb{E}\left[\left\|\nabla F\left(\Phi^{t-1}\right)\right\|^2\right] \leq & \left(\mathbb{E}\left[F\left(z^{t-1}\right)\right]-\mathbb{E}\left[F\left(z^t\right)\right]\right)+\frac{\eta K L^2}{2(1-\lambda)} \frac{1}{K N} \sum_{k, C_i} \mathbb{E}\left[\left\|\boldsymbol{\theta}_{i, k-1+1/2}^t-\boldsymbol{\theta}_{i, 0+1/4}^t\right\|^2\right] \\
	& +\frac{L}{2(1-\lambda)} \mathbb{E}\left[\left\|m^{t-1}\right\|^2\right]+\left(\frac{L}{2(1-\lambda)^3}+\frac{L}{2(1-\lambda)^2}\right) \mathbb{E}\left[\left\|\boldsymbol{\Delta}^t\right\|^2\right]
\end{aligned}
$$
Summing the above inequality for $t \in\{1, \ldots, T\}$ yields
\begin{equation}\label{eq C3}
\begin{aligned}
	\frac{\eta K}{2(1-\lambda)} \sum_{t=1}^T \mathbb{E}\left[\left\|\nabla F\left(\Phi^{t-1}\right)\right\|^2\right] \leq & \left(\mathbb{E}\left[F\left(z^0\right)\right]-\mathbb{E}\left[F\left(z^T\right)\right]\right)+\frac{\eta K L^2}{2(1-\lambda)} \sum_{t=1}^T \frac{1}{K N} \sum_{k, C_i} \mathbb{E}\left[\left\|\boldsymbol{\theta}_{i, k-1+1/2}^t-\boldsymbol{\theta}_{i, 0+1/4}^t\right\|^2\right] \\
	& +\frac{L}{2(1-\lambda)} \sum_{t=1}^T \mathbb{E}\left[\left\|m^{t-1}\right\|^2\right]+\left(\frac{L}{2(1-\lambda)^3}+\frac{L}{2(1-\lambda)^2}\right) \sum_{t=1}^T \mathbb{E}\left[\left\|\boldsymbol{\Delta}^t\right\|^2\right]
\end{aligned}
\end{equation}
By applying Lemma \ref{CL4}, Lemma \ref{CL5}, and Lemma \ref{CL6}, we have

$$
\begin{aligned}
	& \frac{\eta K}{2(1-\lambda)} \sum_{t=1}^T \mathbb{E}\left[\left\|\nabla F\left(\Phi^{t-1}\right)\right\|^2\right] \leq\left(\mathbb{E}\left[F\left(z^0\right)\right]-\mathbb{E}\left[F\left(z^T\right)\right]\right)+\frac{\eta K L^2}{2(1-\lambda)} \sum_{t=1}^T \frac{1}{K N} \sum_{k, C_i} \mathbb{E}\left[\left\|\boldsymbol{\theta}_{i, k-1+1/2}^t-\boldsymbol{\theta}_{i, 1/4}^t\right\|^2\right] \\
	& +\left(\frac{2 L}{2(1-\lambda)^3}+\frac{L}{2(1-\lambda)^2}\right) \sum_{t=1}^T \mathbb{E}\left[\left\|\boldsymbol{\delta}^t\right\|^2\right] \\
	& \leq\left(\mathbb{E}\left[F\left(z^0\right)\right]-\mathbb{E}\left[F\left(z^T\right)\right]\right) \\
	& +\frac{\eta K}{2(1-\lambda)} L^2\left(\frac{8 \eta K L}{(1-\lambda)^2}+\frac{4 \eta K L}{1-\lambda}+1\right) \sum_{t=1}^T \frac{1}{K N} \sum_{k, C_i} \mathbb{E}\left[\left\|\boldsymbol{\theta}_{i, k-1+1/2}^t-\boldsymbol{\theta}_{i, 0+1/4}^t\right\|^2\right] \\
	& +\frac{\eta K}{2(1-\lambda)}\left(\frac{4 \eta K L}{(1-\lambda)^2}+\frac{2 \eta K L}{1-\lambda}\right) \sum_{t=1}^T\left(8\left\|\nabla F\left(\Phi^{t-1}\right)\right\|^2+\frac{4\left(1-\frac{\left|S_t\right|}{N}\right)}{\left|S_t\right|} \sigma_g^2+\frac{\sigma^2}{K\left|S_t\right|}\right) \\
	& \leq\left(\mathbb{E}\left[F\left(z^0\right)\right]-\mathbb{E}\left[F\left(z^T\right)\right]\right) \\
	& +\frac{\eta K}{2(1-\lambda)} L^2\left(\frac{8 \eta K L}{(1-\lambda)^2}+\frac{4 \eta K L}{1-\lambda}+1\right) \sum_{t=1}^T\left(24 \eta^2 K^2\left(\sigma_g^2+ \mathbb{E}\left[\left\|\nabla F\left(\Phi^{t-1}\right)\right\|^2\right]\right)+24 K L^2 \rho^2+6 \eta^2 K \sigma^2+2 \rho^2\right) \\
	& +\frac{\eta K}{2(1-\lambda)}\left(\frac{4 \eta K L}{(1-\lambda)^2}+\frac{2 \eta K L}{1-\lambda}\right) \sum_{t=1}^T\left(8\left\|\nabla F\left(\Phi^{t-1}\right)\right\|^2+\frac{4\left(1-\frac{\left|S_t\right|}{N}\right)}{\left|S_t\right|} \sigma_g^2+\frac{\sigma^2}{K\left|S_t\right|}\right) .
\end{aligned}
$$
$$
\begin{aligned}
	&\text { If } \eta \leq \frac{(1-\lambda)^2}{128 K L} \text {, we can rewrite the above inequality as follows }\\
	&\begin{aligned}
		&\frac{\eta K}{4(1-\lambda)} \sum_{t=1}^T \mathbb{E}\left[\left\|\nabla F\left(\Phi^{t-1}\right)\right\|^2\right]\\
		&\leq  \left(\mathbb{E}\left[F\left(z^0\right)\right]-\mathbb{E}\left[F\left(z^T\right)\right]\right)+\frac{\eta K}{2(1-\lambda)} L^2\left(\frac{8 \eta K L}{(1-\lambda)^2}+\frac{4 \eta K L}{1-\lambda}+1\right) \sum_{t=1}^T\left(24 \eta^2 K^2 \sigma_g^2+24 K L^2 \rho^2+6 \eta^2 K \sigma^2+2 \rho^2\right) \\
		& +\frac{\eta K}{2(1-\lambda)}\left(\frac{4 \eta K L}{(1-\lambda)^2}+\frac{2 \eta K L}{1-\lambda}\right) \sum_{t=1}^T\left(\frac{4\left(1-\frac{\left|S_t\right|}{N}\right)}{\left|S_t\right|} \sigma_g^2+\frac{\sigma^2}{K\left|S_t\right|}\right) \\
		\leq & \left(\mathbb{E}\left[F\left(z^0\right)\right]-\mathbb{E}\left[F\left(z^T\right)\right]\right)+35 L^2(1-\lambda)^5\left(\frac{\eta K}{4(1-\lambda)^2}\right)^3 \sum_{t=1}^T\left(6 \sigma_g^2+\frac{3}{K} \sigma^2\right)\\
		&+35 L^2(1-\lambda)\left(\frac{\eta K}{4(1-\lambda)^2}\right) \sum_{t=1}^T\left(24 K L^2 \rho^2+2 \rho^2 \right) \\
		& +8 L(1-\lambda)\left(\frac{\eta K}{4(1-\lambda)^2}\right)^2 \sum_{t=1}^T\left(\frac{4\left(1-\frac{\left|S_t\right|}{N}\right)}{\left|S_t\right|} \sigma_g^2+\frac{\sigma^2}{K\left|S_t\right|}\right)
	\end{aligned}
\end{aligned}
$$
Let $\tilde{\eta}=\frac{\eta K}{4(1-\lambda)^2}$. By dividing both sides by $1-\lambda$, we have

$$
\begin{aligned}
	\tilde{\eta} \sum_{t=1}^T \mathbb{E}\left[\left\|\nabla F\left(\Phi^{t-1}\right)\right\|^2\right] \leq & \frac{\left(\mathbb{E}\left[F\left(z^0\right)\right]-\mathbb{E}\left[F\left(z^T\right)\right]\right)}{1-\lambda}+35 L^2(1-\lambda)^4 \tilde{\eta}^3 T\left(6 \sigma_g^2+\frac{3}{K} \sigma^2\right) \\
	& +8 L \tilde{\eta}^2 T\left(\frac{4\left(1-\frac{\left|S_t\right|}{N}\right)}{\left|S_t\right|} \sigma_g^2+\frac{\sigma^2}{K\left|S_t\right|}\right)+35 L^2(1-\lambda)\tilde{\eta} T\left(24 K L^2 \rho^2+2 \rho^2 \right) 
\end{aligned}
$$

Dividing both side by $\tilde{\eta} T$ yields

$$
\begin{aligned}
	\frac{1}{T} \sum_{t=1}^T \mathbb{E}\left[\left\|\nabla F\left(\Phi^{t-1}\right)\right\|^2\right] \leq & \frac{\left(\mathbb{E}\left[F\left(\boldsymbol{\theta}^0\right)\right]-\mathbb{E}\left[F\left(\boldsymbol{\theta}^*\right)\right]\right)}{\tilde{\eta} T(1-\lambda)}+35 L^2(1-\lambda)^4 \tilde{\eta}^2\left(6 \sigma_g^2+\frac{3}{K} \sigma^2\right) \\
	& +8 L \tilde{\eta}\left(\frac{4\left(1-\frac{\left|S_t\right|}{N}\right)}{\left|S_t\right|} \sigma_g^2+\frac{\sigma^2}{K\left|S_t\right|}\right)+35 L^2(1-\lambda)\left(24 K L^2 \rho^2+2 \rho^2 \right) 
\end{aligned}
$$
Now we get the desired rate by applying Lemma \ref{CL2}, which finishes the proof.
\end{proof}
\begin{lemma}\label{CL4}
 Algorithm 1 satisfies
	$$
	\sum_{t=1}^T \mathbb{E}\left[\left\|\boldsymbol{m}^t\right\|^2\right] \leq \frac{1}{(1-\lambda)^2} \sum_{t=1}^T \mathbb{E}\left[\left\|\boldsymbol{\Delta}^t\right\|^2\right]
	$$
\end{lemma}
\begin{proof}
Unrolling the recursion of the momentum $\boldsymbol{m}^t$, i.e., $\boldsymbol{m}^t=\sum_{r=1}^t \lambda^{t-r} \boldsymbol{\Delta}^r$
$$
\mathbb{E}\left[\left\|\boldsymbol{m}^t\right\|^2\right]=\mathbb{E}\left[\left\|\sum_{r=1}^t \lambda^{t-r} \boldsymbol{\Delta}^r\right\|^2\right]
$$
Let $\Gamma_t=\sum_{r=0}^{t-1} \lambda^r=\frac{1-\lambda^t}{1-\lambda}$. Since $0 \leq \lambda<1, \Gamma_t \leq \frac{1}{1-\lambda}$, we have
$$
\begin{aligned}
	\mathbb{E}\left[\left\|\sum_{r=1}^t \lambda^{t-r} \boldsymbol{\Delta}^r\right\|^2\right] & =\Gamma_t^2 \mathbb{E}\left[\left\|\frac{1}{\Gamma_t} \sum_{r=1}^t \lambda^{t-r} \boldsymbol{\Delta}^r\right\|^2\right] \\
	& \leq \Gamma_t \sum_{r=1}^t \lambda^{t-r} \mathbb{E}\left[\left\|\boldsymbol{\Delta}^r\right\|^2\right] \\
	& \leq \frac{1}{1-\lambda} \sum_{r=1}^t \lambda^{t-r} \mathbb{E}\left[\left\|\boldsymbol{\Delta}^r\right\|^2\right]
\end{aligned}
$$
By summing the above inequality for $t \in\{0, \ldots, T-1\}$, we have
$$
\begin{aligned}
	\sum_{t=1}^T \mathbb{E}\left[\left\|\boldsymbol{m}^t\right\|^2\right] & \leq \sum_{t=1}^T \frac{1}{1-\lambda} \sum_{r=1}^t \lambda^{t-r} \mathbb{E}\left[\left\|\boldsymbol{\Delta}^r\right\|^2\right] \\
	& \leq \frac{1}{(1-\lambda)^2} \sum_{t=1}^T \mathbb{E}\left[\left\|\boldsymbol{\Delta}^t\right\|^2\right]
\end{aligned}
$$
which finishes the proof.
\end{proof}
\begin{lemma}\label{CL5} For all $t \geq 1$, Algorithm 1 satisfies
	$$
\mathbb{E}\left[\left\|\boldsymbol{\Delta}^t\right\|^2\right] \leq 2(\eta K)^2\left(\frac{2 L^2}{K N} \sum_{k, C_i} \mathbb{E}\left[\left\|\boldsymbol{\theta}_{i, k+1/2}^t-\boldsymbol{\theta}_{i, 0+1/4}^t\right\|^2\right]+8\left\|\nabla F\left(\Phi^{t-1}\right)\right\|^2+\frac{4\left(1-\frac{\left|S_t\right|}{N}\right)}{\left|S_t\right|} \sigma_g^2+\frac{\sigma^2}{K\left|S_t\right|}\right)
	$$
	, where $\boldsymbol{\theta}_{i, 0}^t$ denotes the initial point for the local model of the $i$-th client, i.e., $\boldsymbol{\theta}_{i, 0+1/4}^t=\Phi^{t-1}$.
\end{lemma}

\begin{proof}
By applying Lemma \ref{CL3}, we have
$$
\begin{aligned}
	\mathbb{E}\left[\left\|\boldsymbol{\Delta}^t\right\|^2\right] & =\mathbb{E}\left[\left\|\frac{\eta K}{K\left|S_t\right|} \sum_{k, C_i \in S_t} \nabla F_i\left(\boldsymbol{\theta}_{i, k+1/2}^t;\zeta_i\right)\right\|^2\right] \\
	& \leq 2(\eta K)^2\left(\mathbb{E}\left[\left\|\frac{1}{K\left|S_t\right|} \sum_{k, C_i \in S_t} \nabla F_i\left(\boldsymbol{\theta}_{i, k+1/2}^t\right)\right\|^2\right]+\frac{\sigma^2}{K\left|S_t\right|}\right)
\end{aligned}
$$
We note that
$$
\begin{aligned}
	& \mathbb{E}\left[\left\|\frac{1}{K\left|S_t\right|} \sum_{k, C_i \in S_t} \nabla F_i\left(\boldsymbol{\theta}_{i, k+1/2}^t\right)\right\|^2\right]    \\
	& =\mathbb{E}\left[\left\|\frac{1}{K\left|S_t\right|} \sum_{k, C_i \in S_t}\left(\nabla F_i\left(\boldsymbol{\theta}_{i, k+1/2}^t\right)-\nabla F_i\left(\boldsymbol{\theta}_{i, 0+1/4}^t\right)+\nabla F_i\left(\boldsymbol{\theta}_{i, 0+1/4}^t\right)\right)\right\|^2\right] \\
	& \leq 2 \mathbb{E}\left[\left\|\frac{1}{K\left|S_t\right|} \sum_{k, C_i \in S_t}\left(\nabla F_i\left(\boldsymbol{\theta}_{i, k+1/2}^t\right)-\nabla F_i\left(\boldsymbol{\theta}_{i, 0+1/4}^t\right)\right)\right\|^2\right]+2 \mathbb{E}\left[\left\|\frac{1}{\left|S_t\right|} \sum_{C_i \in S_t} \nabla F_i\left(\boldsymbol{\theta}_{i, 0+1/4}^t\right)\right\|^2\right] \\
	& \leq \frac{2}{K N} \sum_{k, C_i} \mathbb{E}\left[\left\|\nabla F_i\left(\boldsymbol{\theta}_{i, k+1/2}^t\right)-\nabla F_i\left(\boldsymbol{\theta}_{i, 0+1/4}^t\right)\right\|^2\right]+\mathbb{E}\left[\left\|\frac{2}{\left|S_t\right|} \sum_{C_i \in S_t}\left(\nabla F_i\left(\boldsymbol{\theta}_{i, 0+1/4}^t\right)-\nabla F\left(\Phi^{t-1}\right)+\nabla F\left(\Phi^{t-1}\right)\right)\right\|^2\right] \\
	& \leq \frac{2 L^2}{K N} \sum_{k, C_i} \mathbb{E}\left[\left\|\boldsymbol{\theta}_{i, k+1/2}^t-\boldsymbol{\theta}_{i, 0+1/4}^t\right\|^2\right]+\mathbb{E}\left[\left\|\frac{2}{\left|S_t\right|} \sum_{C_i \in S_t}\left(\nabla F_i\left(\boldsymbol{\theta}_{i, 0+1/4}^t\right)-\nabla F\left(\Phi^{t-1}\right)+\nabla F\left(\Phi^{t-1}\right)\right)\right\|^2\right] \\
	& \leq \frac{2 L^2}{K N} \sum_{k, C_i} \mathbb{E}\left[\left\|\boldsymbol{\theta}_{i, k+1/2}^t-\boldsymbol{\theta}_{i, 0+1/4}^t\right\|^2\right]+4\left\|\nabla F\left(\Phi^{t-1}\right)\right\|^2+\frac{4\left(1-\frac{\left|S_t\right|}{N}\right)}{\left|S_t\right| N} \sum_{C_i}\left\|\nabla F_i\left(\boldsymbol{\theta}_{i, 0+1/4}^t\right)\right\|^2 \\
	& \leq \frac{2 L^2}{K N} \sum_{k, C_i} \mathbb{E}\left[\left\|\boldsymbol{\theta}_{i, k+1/2}^t-\boldsymbol{\theta}_{i, 0+1/4}^t\right\|^2\right]+8\left\|\nabla F\left(\Phi^{t-1}\right)\right\|^2+\frac{4\left(1-\frac{\left|S_t\right|}{N}\right)}{\left|S_t\right|} \sigma_g^2
\end{aligned}
$$
where, in the fourth inequality, the improvement of $\left(1-\frac{\left|S_t\right|}{N}\right)$ follows from sampling the active client set $S_t$ without replacement at the $t$-th communication round. The last inequality holds because the average norm of local gradients is bounded as $\frac{1}{N} \sum_{i=1}^N\left\|\nabla F_i(x)\right\|^2 \leq \sigma_g^2+\|\nabla F(x)\|^2$, which concludes the proof.
\end{proof}

\begin{lemma}\label{CL6}
	For all $t \geq 1$, we have
	$$
	\frac{1}{K N} \sum_{k, C_i} \mathbb{E}\left[\left\|\boldsymbol{\theta}_{i, k+1/2}^t-\boldsymbol{\theta}_{i, 0+1/4}^t\right\|^2\right]  \leq 24 \eta^2 K^2\left(\sigma_g^2+ \mathbb{E}\left[\left\|\nabla F\left(\Phi^{t-1}\right)\right\|^2\right]\right)+24 K L^2 \rho^2+6 \eta^2 K \sigma^2+2 \rho^2
	$$
\end{lemma}

\begin{proof}
 We start with a one-step conversion
$$
\begin{aligned}
	& \mathbb{E}\left[\left\|\boldsymbol{\theta}_{i, k+1 / 2}^t-\boldsymbol{\theta}_{i, 0+1 / 4}^t\right\|^2\right] \\
	& =\mathbb{E}\left[\left\|\boldsymbol{\theta}_{i, k}^t+\lambda \boldsymbol{m}^t-\rho \frac{\boldsymbol{m}^t}{\left\|\boldsymbol{m}^t\right\|}-\boldsymbol{\theta}_{i, 0+1 / 4}^t\right\|^2\right] \\
	& \leq \mathbb{E}\left[\left\|\boldsymbol{\theta}_{i, k}^t-\rho \frac{\boldsymbol{m}^t}{\left\|\boldsymbol{m}^t\right\|}-\boldsymbol{\theta}_{i, 0}^t\right\|^2\right] \\
	& \leq 2 \mathbb{E}\left[\left\|\boldsymbol{\theta}_{i, k}^t-\boldsymbol{\theta}_{i, 0}^t\right\|^2\right]+2 \rho^2
\end{aligned}
$$

$$
\begin{aligned}
	& \frac{1}{K N} \sum_{k, C_i} \mathbb{E}\left[\left\|\boldsymbol{\theta}_{i, k+1 / 2}^t-\boldsymbol{\theta}_{i, 0+1 / 4}^t\right\|^2\right] \\
	& \leq \frac{2}{K N} \sum_{k, C_i}\left[\mathbb{E}\left[\left\|\boldsymbol{\theta}_{i, k}^t-\boldsymbol{\theta}_{i, 0}^t\right\|^2\right]+2 \rho^2\right.
\end{aligned}
$$
We first define the following terms as
\begin{equation}\label{CE8}
I_{i, k}^t=\mathbb{E}\left[\left\|\boldsymbol{\theta}_{i, k}^t-\boldsymbol{\theta}_{i, 0}^t\right\|^2\right], I_i^t=\frac{1}{K} \sum_{k=1}^K I_{i, k}^t, I^t=\frac{1}{N} \sum_{C_i} I_i^t
\end{equation}
Initially, we commence by deriving an upper bound for the variable $I_{i, k}^t$ as
\begin{equation}\label{CE9}
\begin{aligned}
	& \mathbb{E}\left[\left\|\boldsymbol{\theta}_{i, k}^t-\boldsymbol{\theta}_{i, 0}^t\right\|^2\right] \\
	& =\mathbb{E}\left[\left\|\boldsymbol{\theta}_{i, k-1}^t-\eta \nabla F_i\left(\boldsymbol{\theta}_{i, k-1+1 / 2}^t;\zeta_i\right)-\boldsymbol{\theta}_{i, 0}^t\right\|^2\right] \\
	& \leq \mathbb{E}\left[\left\|\boldsymbol{\theta}_{i, k-1}^t-\boldsymbol{\theta}_{i, 0}^t-\eta \nabla F_i\left(\boldsymbol{\theta}_{i, k-1+1 / 2}^t\right)\right\|^2\right]+\eta^2 \sigma^2 \\
	& \leq\left(1+\frac{1}{K-1}\right) \mathbb{E}\left[\left\|\boldsymbol{\theta}_{i, k-1}^t-\boldsymbol{\theta}_{i, 0}^t\right\|^2\right]+K \eta^2 \mathbb{E}\left[\left\|\nabla F_i\left(\boldsymbol{\theta}_{i, k-1+1 / 2}^t\right)\right\|^2\right]+\eta^2 \sigma
\end{aligned}
\end{equation}
where the first inequality follows because the stochastic gradient possesses a bounded variance, while the second inequality follows from the Lemma \ref{CL1}.

We note that
\begin{equation}\label{CE10}
\begin{aligned}
	& \mathbb{E}\left[\left\|\nabla F_i\left(\boldsymbol{\theta}_{i, k-1+1 / 2}^{t}\right)\right\|^2\right]=\mathbb{E}\left[\left\|\nabla F_i\left(\boldsymbol{\theta}_{i, k-1+1 / 2}^t\right)-\nabla F_i\left(\boldsymbol{\theta}_{i, 1 / 2}^t\right)+\nabla F_i\left(\boldsymbol{\theta}_{i, 1 / 2}^t\right)\right\|^2\right] \\
	& \leq 2 \mathbb{E}\left[\left\|\nabla F_i\left(\boldsymbol{\theta}_{i, k-1+1 / 2}^t\right)-\nabla F_i\left(\boldsymbol{\theta}_{i, 1 / 2}^t\right)\right\|^2\right]+2 \mathbb{E}\left[\left\|\nabla F_i\left(\boldsymbol{\theta}_{i, 1 / 2}^t\right)\right\|^2\right] \\
	& \leq 2 \mathbb{E}\left[\left\|\boldsymbol{\theta}_{i, k-1+1 / 2}^t-\boldsymbol{\theta}_{i, 1 / 2}^t\right\|^2\right]+2 \mathbb{E}\left[\left\|\nabla F_i\left(\boldsymbol{\theta}_{i, 1 / 2}^t\right)+\nabla F_i\left(\boldsymbol{\theta}_{i, 1 / 4}^{\prime}\right)-\nabla F_i\left(\boldsymbol{\theta}_{i, 1 / 4}^t\right)\right\|^2\right] \\
	& \leq 2 L^2 \mathbb{E}\left[\left\|\boldsymbol{\theta}_{i, k-1}^t-\boldsymbol{\theta}_{i, 0}^t\right\|^2\right]+4 \mathbb{E}\left[\left\|\nabla F_i\left(\boldsymbol{\theta}_{i, 1 / 4}^t\right)\right\|^2\right]+4 L^2 \rho^2
\end{aligned}
\end{equation}

By substituting Eq. (\ref{CE9}) into Eq. (\ref{CE10}), we have

$$
\begin{aligned}
	I_{i, k}^t & \leq\left(1+\frac{1}{K-1}+2 K \eta^2 L^2\right) \mathbb{E}\left[\left\|\boldsymbol{\theta}_{i, k-1}^t-\boldsymbol{\theta}_{i, 0}^t\right\|^2\right]+4 K \eta^2 \mathbb{E}\left[\left\|\nabla F_i\left(\boldsymbol{\theta}_{i, 0}^t\right)\right\|^2\right]+4 L^2 \rho^2+\eta^2 \sigma^2 \\
	& \leq\left(1+\frac{1}{K-1}+2 K \eta^2 L^2\right) I_{i, k-1}^t+4 K \eta^2 \mathbb{E}\left[\left\|\nabla F_i\left(\boldsymbol{\theta}_{i, 0}^t\right)\right\|^2\right]+4 L^2 \rho^2+\eta^2 \sigma^2
\end{aligned}
$$

By unrolling the recursion, we have

$$
I_{i, k}^t \leq \sum_{r=0}^{k-1}\left(4 K \eta^2 \mathbb{E}\left[\left\|\nabla F_i\left(\boldsymbol{\theta}_{i, 0}^t\right)\right\|^2\right]+4 L^2 \rho^2+\eta^2 \sigma^2\right)\left(1+\frac{2}{K-1}\right)^r \leq 3 K\left(4 K \eta^2 \mathbb{E}\left[\left\|\nabla F_i\left(\boldsymbol{\theta}_{i, 0}^t\right)\right\|^2\right]+4 L^2 \rho^2+\eta^2 \sigma^2\right)
$$

By the definitions in Eq.(\ref{CE8}) , we have

$$
\begin{aligned}
	I_i^t & =\frac{1}{K} \sum_{k=1}^K I_{i, k}^t \leq 3 K\left(4 K \eta^2 \mathbb{E}\left[\left\|\nabla F_i\left(\boldsymbol{\theta}_{i, 0}^t\right)\right\|^2\right]+4 L^2 \rho^2+\eta^2 \sigma^2\right) \\
	& =12 \eta^2 K^2 \mathbb{E}\left[\left\|\nabla F_i\left(\boldsymbol{\theta}_{i, 0}^t\right)\right\|^2\right]+12 K L^2 \rho^2+3 \eta^2 K \sigma^2 \\
	I^t & =12 \eta^2 K^2 \frac{1}{N} \sum_{C_i} \mathbb{E}\left[\left\|\nabla F_i\left(\boldsymbol{\theta}_{i, 0}^t\right)\right\|^2\right]+12 K L^2 \rho^2+3 \eta^2 K \sigma^2 \\
	& \leq 12 \eta^2 K^2\left(\sigma_g^2+ \mathbb{E}\left[\left\|\nabla F\left(\Phi^{t-1}\right)\right\|^2\right]\right)+12 K L^2 \rho^2+3 \eta^2 K \sigma^2
\end{aligned}
$$
where the inequality follows due to the assumption that the average norm of the local gradients is bounded, i.e., $\frac{1}{N} \sum_{i=1}^N\left\|\nabla F_i(x)\right\|^2 \leq \sigma_g^2+\|\nabla F(x)\|^2$, which completes the proof.
$$
\begin{aligned}
	& \frac{1}{K N} \sum_{k, C_i} \mathbb{E}\left[\left\|\boldsymbol{\theta}_{i, k+1 / 2}^t-\boldsymbol{\theta}_{i, 0+1 / 4}^t\right\|^2\right] \leq \frac{2}{K N} \sum_{k, C_i} \mathbb{E}\left[\left\|\boldsymbol{\theta}_{i, k+1 / 2}^t-\boldsymbol{\theta}_{i, 0+1 / 4}^t\right\|^2\right]+2 \rho^2 \\
	& \leq 2\left[12 \eta^2 K^2\left(\sigma_g^2+ \mathbb{E}\left[\left\|\nabla F\left(\Phi^{t-1}\right)\right\|^2\right]\right)+12 K L^2 \rho^2+3 \eta^2 K \sigma^2\right]+2 \rho^2 \\
	& \leq 24 \eta^2 K^2\left(\sigma_g^2+ \mathbb{E}\left[\left\|\nabla F\left(\Phi^{t-1}\right)\right\|^2\right]\right)+24 K L^2 \rho^2+6 \eta^2 K \sigma^2+2 \rho^2
\end{aligned}
$$
\end{proof}
\begin{lemma} (Bounded $\mathcal{E}_w$ of FedSAM). Suppose our functions satisfies Assumptions 1-2. Then, the updates of FedSAM for any learning rate satisfying $\eta \leq \frac{1}{10 K L}$ have the drift due to $\theta_{i, k}^t-\theta^t$ :
\end{lemma}

$$
\mathcal{E}_w=\frac{1}{N} \sum_i \mathbb{E}\left[\left\|\theta_{i, K}^t-\theta^t\right\|^2\right] \leq 5 K \eta^2\left(2 L^2 \rho^2 \sigma^2+6 K\left(3 \sigma_g^2+6 L^2 \rho^2\right)+6 K\|\nabla f(\tilde{\theta}^t)\|^2\right)+24 K^3 \eta^4 L^4 \rho^2
$$
\begin{proof}
Recall that the local update on client $i$ is $\theta_{i, k}^t=\theta_{i, k-1}^t-\eta \tilde{g}_{i, k-1}^t$. Then,

$$
\begin{aligned}
	& \mathbb{E}\left\|\theta_{i, k}^t-\theta^t\right\|^2=\mathbb{E}\left\|\theta_{i, k-1}^t-\theta^t-\eta \tilde{g}_{i, k-1}^t\right\|^2 \\
	& \stackrel{(a)}{\leq} \mathbb{E}\left\|\theta_{i, k-1}^t-\theta^t-\eta_l\left(\tilde{g}_{i, k-1}-\nabla f_i\left(\tilde{\theta}_{i, k-1}^t\right)+\nabla f_i\left(\tilde{\theta}_{i, k-1}^t\right)-\nabla f_i(\tilde{\theta}^t)+\nabla f_i(\tilde{\theta}^t)\right)-\nabla f(\tilde{\theta}^t)+\nabla f(\tilde{\theta}^t)\right\|^2 \\
	& \stackrel{(b)}{\leq}\left(1+\frac{1}{2 K-1}\right) \mathbb{E}\left\|\theta_{i, k-1}^t-\theta^t\right\|^2+\mathbb{E}\left\|\eta\left(\tilde{g}_{i, k-1}^t-\nabla f_i\left(\tilde{\theta}_{i, k-1}^t\right)\right)\right\|^2 \\
	& +6 K \mathbb{E}\left\|\eta_l\left(\nabla f_i\left(\tilde{\theta}_{i, k-1}^t\right)-\nabla f_i(\tilde{\theta}^t)\right)\right\|^2+6 K \mathbb{E}\left\|\eta\left(\nabla f_i(\tilde{\theta}^t)-\nabla f(\tilde{\theta}^t)\right)\right\|^2+6 K\left\|\eta\nabla f(\tilde{\theta}^t)\right\|^2 \\
	& \stackrel{\text { (c) }}{\leq}\left(1+\frac{1}{2 K-1}+2 L^2 \eta^2\right) \mathbb{E}\left\|\theta_{i, k-1}^t-\theta^t\right\|^2+2 \eta^2 L^2 \rho^2 \sigma^2+12 K \eta^2 L^2 \mathbb{E}\left\|\theta_{i, k-1}^t-\theta^t\right\|^2 \\
	& +12 K L^2 \eta_l^2 \mathbb{E}\left\|\delta_{i, k-1}^t-\delta^t\right\|^2+6 K \eta^2 \mathbb{E}\left\|\nabla f_i(\tilde{\theta}^t)-\nabla f(\tilde{\theta}^t)\right\|^2+6 K\|\nabla f(\tilde{\theta}^t)\|^2 \\
	& \stackrel{\text { (d) }}{\leq}\left(1+\frac{1}{2 K-1}+12 K \eta_l^2 L^2+2 L^2 \eta^2\right) \mathbb{E}\left\|\theta_{i, k-1}^t-\theta^t\right\|^2+2 \eta^2 L^2 \rho^2 \sigma_l^2+12 K L^2 \eta_l^2 \mathbb{E}\left\|\delta_{i, k}^t-\delta^t\right\|^2 \\
	& +6 K \eta^2\left(3 \sigma_g^2+6 L^2 \rho^2\right)+6 K\|\nabla f(\tilde{\theta}^t)\|^2,
\end{aligned}
$$
where (a) follows from the fact that $\tilde{g}_{i, k-1}^t$ is an unbiased estimator of $\nabla f_i\left(\tilde{\theta}_{i, k-1}^t\right)$ and Lemma A.3; (b) is from Lemma A.2; (c) is from Assumption 3 and Lemma A. 2 and (d) is from Lemma A.5.

Averaging over the clients $i$ and learning rate satisfies $\eta \leq \frac{1}{10 K L}$, we have
$$
\begin{aligned}
	&\begin{aligned}
		& \frac{1}{N} \sum_{i \in[N]} \mathbb{E}\left\|\theta_{i, k}^t-\theta^t\right\|^2 \leq\left(1+\frac{1}{2 K-1}+12 K \eta^2 L^2+2 L^2 \eta_l^2\right) \frac{1}{N} \sum_{i \in[N]} \mathbb{E}\left\|\theta_{i, k-1}^t-\theta^t\right\|^2 \\
		& \quad+2 \eta^2 L^2 \rho^2 \sigma^2+12 K L^2 \eta^2 \frac{1}{N} \sum_{i \in[N]} \mathbb{E}\left\|\delta_{i, k}^t-\delta^t\right\|^2+6 K \eta^2\left(3 \sigma_g^2+6 L^2 \rho^2\right)+6 K\|\nabla f(\tilde{\theta}^t)\|^2 \\
		& \quad \stackrel{(a)}{\leq}\left(1+\frac{1}{K-1}\right) \frac{1}{N} \sum_{i \in[N]} \mathbb{E}\left\|\theta_{i, k-1}^t-\theta^t\right\|^2+\eta^2 L^2 \rho^2 \sigma^2 \\
		& \quad+12 K L^2 \eta_l^2 \frac{1}{N} \sum_{i \in[N]} \mathbb{E}\left\|\delta_{i, k}^t-\delta^t\right\|^2+6 K \eta^2\left(3 \sigma_g^2+6 L^2 \rho^2\right)+6 K\|\nabla f(\tilde{\theta}^t)\|^2 \\
		& \quad \leq \sum_{\tau=0}^{k-1}\left(1+\frac{1}{K-1}\right)^\tau\left[2 \eta^2 L^2 \rho^2 \sigma^2+6 K \eta^2\left(3 \sigma_g^2+6 L^2 \rho^2\right)+6 K\|\nabla f(\tilde{\theta}^t)\|^2\right]+12 K L^2 \eta^2 \frac{1}{N} \sum_{i \in[N]} \mathbb{E}\left\|\delta_{i, k}^t-\delta^t\right\|^2 \\
		& \quad \stackrel{\text { (b) }}{\leq} 5 K \eta^2\left(2 L^2 \rho^2 \sigma^2+6 K\left(3 \sigma_g^2+6 L^2 \rho^2\right)+6 K\|\nabla f(\tilde{\theta}^t)\|^2\right)+24 K^3 \eta^4 L^4 \rho^2,
	\end{aligned}\\
	&\text { where (a) is due to the fact that } \eta \leq \frac{1}{10 K L} \text { and (b) is from Lemma B.1. }
\end{aligned}
$$
\end{proof}

\begin{lemma}(Bounded $\mathcal{E}_\delta$ of FedSAM \cite{qu2022generalized}). Suppose our functions satisfy Assumptions 1-2. Then, the updates of FedSAM for any learning rate satisfying $\eta \leq \frac{1}{4 K L}$ have the drift due to $\delta_{i, k}^t-\delta^t$ :
\end{lemma}
$$
\mathcal{E}_\delta=\frac{1}{N} \sum_i \mathbb{E}\left[\left\|\delta_{i, k}^t-\delta^t\right\|^2\right] \leq 2 K^2 \beta^2 \eta^2 \rho^2 .
$$
Recall the definitions of $\delta^t$ and $\delta_{i, k}^t$ as follows:
$$
\delta^t=\rho \frac{\nabla F(\theta^t)}{\|\nabla F(\theta^t)\|}, \quad \delta_{i, k}^t=\rho \frac{\nabla F_i\left(\theta_{i, k}^t, \xi_i\right)}{\left\|\nabla F_i\left(\theta_{i, k}^t, \xi_i\right)\right\|}
$$

\begin{lemma} For the full client participation scheme, we can bound $\mathbb{E}\left[\left\|\Delta^t\right\|^2\right]$ as follows:
	
	$$
\mathbb{E}_r\left[\left\|\Delta^t\right\|^2\right] \leq \frac{K \eta^2 L^2 \rho^2}{N} \sigma^2+\frac{\eta^2}{N^2}\left[\left\|\sum_{i, k} \nabla f_i\left(\tilde{\theta}_{i, k}^t\right)\right\|^2\right] .
	$$
\end{lemma}

\begin{proof}
 For the full client participation scheme, we have:

$$
\begin{aligned}
	\mathbb{E}_t\left[\left\|\Delta^t\right\|^2\right] & \stackrel{(\text { a })}{\leq} \frac{\eta_l^2}{N^2} \mathbb{E}_t\left[\left\|\sum_{i, k} \tilde{g}_{i, k}^t\right\|^2\right] \stackrel{(\mathrm{b})}{=} \frac{\eta_l^2}{N^2} \mathbb{E}_t\left[\left\|\sum_{i, k}\left(\tilde{g}_{i, k}^t-\nabla f_i\left(\tilde{\theta}_{i, k}^t\right)\right)\right\|^2\right]+\frac{\eta_l^2}{N^2} \mathbb{E}_t\left[\left\|\sum_{i, k} \nabla f_i\left(\tilde{\theta}_{i, k}^t\right)\right\|^2\right] \\
	& \stackrel{(\mathrm{c})}{\leq} \frac{K \eta^2 L^2 \rho^2}{N} \sigma^2+\frac{\eta^2}{N^2}\left[\left\|\sum_{i, k} \nabla f_i\left(\tilde{\theta}_{i, k}^t\right)\right\|^2\right],
\end{aligned}
$$

where (a) is from Lemma A.2; (b) is from Lemma A. 3 and (c) is from Lemma A.4 in \cite{qu2022generalized}.

$$
\begin{aligned}
	& \frac{\eta^2}{N^2} \mathbb{E}\left\|\sum_{i=1}^N \sum_{k=1}^K \nabla f_i\left(\tilde{\theta}_{i, k}^t\right)\right\|^2 \\
	& =\eta^2 \mathbb{E}\left\|\frac{1}{N} \sum_{i=1}^N \sum_{k=1}^K \nabla f_i\left(\tilde{\theta}_{i, k}^t\right)-\frac{1}{N} \sum_{i=1}^N \sum_{k=1}^K \nabla f_i\left(\tilde{\theta}^t\right)+\frac{1}{N} \sum_{i=1}^N \sum_{k=1}^K \nabla f_i\left(\tilde{\theta}^t\right)\right\|^2 \\
	& =\eta^2 K^2 \mathbb{E}\left\|\frac{1}{N K} \sum_{i=1}^N \sum_{k=1}^K \nabla f_i\left(\tilde{\theta}_{i, k}^t\right)-\frac{1}{N K} \sum_{i=1}^N \sum_{k=1}^K \nabla f_i\left(\tilde{\theta}^t\right)+\frac{1}{N K} \sum_{i=1}^N \sum_{k=1}^K \nabla f_i\left(\tilde{\theta}^t\right)\right\|^2 \\
	& \leq 2 \eta^2 K^2 \mathbb{E}\left\|\frac{1}{N K} \sum_{i=1}^N \sum_{k=1}^K \nabla f_i\left(\tilde{\theta}_{i, k}^t\right)-\frac{1}{N K} \sum_{i=1}^N \sum_{k=1}^K \nabla f_i\left(\tilde{\theta}^t\right)\right\|^2+2 \eta^2 K^2 \mathbb{E}\left\|\frac{1}{N K} \sum_{i=1}^N \sum_{k=1}^K \nabla f_i\left(\tilde{\theta}^t\right)\right\| \|^2 \\
	& \leq 2 \eta^2 L^2 K^2 \frac{1}{N K} \sum_{i=1}^N \sum_{k=1}^K \mathbb{E}\left\|\tilde{\theta}_{i, k}^t-\tilde{\theta}^t\right\|^2+2 \eta^2 K^2 \mathbb{E}\left\|\nabla f\left(\tilde{\theta}^t\right)\right\|^2 \\
	& \leq 2 \eta^2 L^2 K^2\left[5 K \eta^2\left(2 L^2 \rho^2 \sigma^2+6 K\left(3 \sigma_g^2+6 L^2 \rho^2\right)+6 K\left\|\nabla f\left(\tilde{\theta}^t\right)\right\|^2\right)+24 K^3 \eta^4 L^4 \rho^2\right]+2 \eta^2 K^2 \mathbb{E}\left\|\nabla f\left(\tilde{\theta}^t\right)\right\|^2
\end{aligned}
$$

$$
\begin{aligned}
	& \mathbb{E}_t\left[\left\|\Delta^t\right\|^2\right] \leq \frac{K \eta^2 L^2 \rho^2}{N} \sigma^2+\frac{\eta^2}{N^2}\left[\left\|\sum_{i, k} \nabla f_i\left(\tilde{\theta}_{i, k}^t\right)\right\|\right. \\
	& \leq \frac{K \eta^2 L^2 \rho^2}{N} \sigma^2+2 \eta^2 L^2 K^2\left[5 K \eta^2\left(2 L^2 \rho^2 \sigma^2+6 K\left(3 \sigma_g^2+6 L^2 \rho^2\right)+6 K\left\|\nabla f\left(\tilde{\theta}^t\right)\right\|^2\right)+24 K^3 \eta^4 L^4 \rho^2\right] \\
	& +2 \eta^2 K^2 \mathbb{E}\left\|\nabla f\left(\tilde{\theta}^t\right)\right\|^2 \\
	& \leq \frac{K \eta^2 L^2 \rho^2}{N} \sigma^2+\left[2 \eta^2 L^2 K^2\left(10 L^2 \rho^2 K \eta^2 \sigma^2+\left(3 \sigma_g^2+6 L^2 \rho^2\right) 30 K^2 \eta^2+60 L^2 K^4 \eta^4\left\|\nabla f\left(\tilde{\theta}^t\right)\right\|^2\right)+24 K^3 \eta^4 L^4 \rho^2\right] \\
	& +2 \eta^2 K^2 \mathbb{E}\left\|\nabla f\left(\tilde{\theta}^t\right)\right\|^2 \\
	& \leq \frac{K \eta^2 L^2 \rho^2}{N} \sigma^2+20 \eta^4 L^4 K^3 \rho^2 \sigma^2+180 \sigma_g^2 L^2 K^4 \eta^4+360 L^4 K^4 \eta^4 \rho^2+60 L^2 K^4 \eta^4\left\|\nabla f\left(\tilde{\theta}^t\right)\right\|^2 \\
	& +48 K^5 \eta^6 L^6 \rho^2+2 \eta^2 K^2 \mathbb{E}\left\|\nabla f\left(\tilde{\theta}^t\right)\right\|^2
\end{aligned}
$$
\end{proof}
\begin{lemma} For the full client participation scheme, we can bound $\mathbb{E}\left[\left\|\Delta^t\right\|^2\right]$ as follows:
	
	$$
	\mathbb{E}_r\left[\frac{1}{N} \sum_{i=1}^N \mathbb{E}\left\|\theta_{i, K}^t-\theta^{t+1}\right\|^2\right] .
	$$
\end{lemma}
\begin{proof}
$$
\begin{aligned}
	& \frac{1}{N} \sum_{i=1}^N \mathbb{E}\left\|\theta_{i, K}^t-\theta^{t+1}\right\|^2=\frac{1}{N} \sum_{i=1}^N\left\|\theta_{i, K}^t-\frac{1}{N} \sum_{i=1}^N \theta_{i, K}^t\right\|^2 \\
	& \leq \frac{2}{N} \sum_{i=1}^N\left\|\theta_{i, K}^t-\theta^t\right\|^2+\frac{2}{N} \sum_{i=1}^N\left\|\theta^{t+1}-\theta^t\right\|^2 \\
	& \leq\left[\frac{2 K \eta^2 L^2 \rho^2}{N}+40 \eta^4 L^4 K^3 \rho^2+20 L^2 \rho^2 K \eta^2\right] \sigma^2+\left[360 L^2 K^4 \eta^4+180 K^2 \eta^2\right] \sigma_g^2 \\
	& +\left[720 L^4 K^4 \eta^4+96 K^5 \eta^6 L^6+360 L^2+48 K^3 \eta^4 L^4\right] \rho^2+\left[120 L^2 K^4 \eta^4+4 \eta^2 K^2+60 K^2 \eta^2\right]\left\|\nabla f\left(\tilde{\theta}^t\right)\right\|^2
\end{aligned}
$$
$
\text { If we set } \eta=\mathcal{O}\left(\frac{1}{\sqrt{T} K L}\right), \rho=\mathcal{O}\left(\frac{1}{\sqrt{T}}\right) \text { used in Theorem } 1 \text {. }
$
$$
\begin{aligned}
	& \leq\left[\frac{2}{K N T^2}+\frac{40}{K T^3}+\frac{20}{K T^2}\right] \sigma^2+\left[\frac{360}{L^2 T^2}+\frac{180}{L^2 T}\right] \sigma_g^2 \\
	& +\left[\frac{720}{T^2}+\frac{96}{K T^3}+360 L^2+\frac{48}{K T^2}\right] \frac{1}{T}+\left[120 \frac{1}{L^2 T^2}+64 \frac{1}{L^2 T}\right]\left\|\nabla f\left(\tilde{\theta}^{\prime}\right)\right\|^2
\end{aligned}
$$
\end{proof}


\end{document}